\documentclass[letterpaper]{article} 
\usepackage[preprint]{aaai2027}  
\usepackage[hyphens]{url}  
\usepackage{graphicx} 
\usepackage{natbib}  
\usepackage{caption} 
\usepackage{enumitem}
\usepackage{amsmath}
\usepackage{amssymb}
\usepackage{amsthm}
\theoremstyle{plain}
\newtheorem{theorem}{Theorem}

\newtheorem{lemma}{Lemma}
\theoremstyle{definition}
\newtheorem{definition}{Definition}

\usepackage{algorithm}
\usepackage{algorithmic}

\usepackage{newfloat}
\usepackage{listings}
\DeclareCaptionStyle{ruled}{labelfont=normalfont,labelsep=colon,strut=off} 
\floatstyle{ruled}
\newfloat{listing}{tb}{lst}{}
\floatname{listing}{Listing}

\usepackage{booktabs}

\DeclareMathOperator{\tr}{tr}

\title{A Probabilistic Circuit-Induced Pseudo-Metric for Out-of-Distribution Detection}
\author{
    Bhumika K,
    Vidhya S,
    Dr Narayanan C Krishnan
}
\affiliations{
    Mehta Family School of Data Science and Artificial Intelligence\\
    Indian Institute of Technology Palakkad\\
    Palakkad, Kerala, India\\
    
    \{142414001,142414002\}@smail.iitpkd.ac.in, ckn@iitpkd.ac.in
}

\begin{document}

\maketitle

\begin{abstract}
Probabilistic Circuits (PCs) are tractable generative models whose internal nodes encode a hierarchy of probabilistic summaries over different variable scopes. Existing PC-based out-of-distribution (OOD) detection methods ignore this hierarchy, reducing the entire circuit to the scalar likelihood (or its uncertainty) computed at the root. We introduce Hierarchical Likelihood Vector (HLV), a representation whose entries are the likelihoods associated with selected PC nodes and define the Hierarchical Likelihood Distance (HLD), a PC-induced pseudo-metric that compares the probability distributions through the expectations of their HLVs. We show that HLD is an integral probability metric over a function class naturally induced by the PC and develop a principled goodness-of-fit hypothesis test for unsupervised OOD detection. Unlike existing approaches, the trained PC alone serves as the representation of the in-distribution: no held-out in-distribution data are required at deployment. We further show that the quantities required by the hypothesis test can be computed exactly, directly from the trained circuit, yielding an approximate analytic decision threshold. Experiments on tabular and MNIST datasets demonstrate that exploiting the hierarchical probabilistic summaries encoded through the PC improve OOD detection over root-likelihood, uncertainty-, typicality- and kernel-based baselines, while naturally localizing distribution shifts to the PC nodes responsible for the shift.

\end{abstract}

\section{Introduction}
Out-of-distribution (OOD) detection seeks to determine whether data encountered at deployment remain consistent with the distribution represented by a learned model. In the unsupervised setting, generative models provide a natural basis for this task because they explicitly model the in-distribution (ID) and assign probabilistic scores to observations. Probabilistic Circuits (PCs) \cite{poo11,rah14,kis14} are particularly attractive in this context \cite{mau17,cho21,cor20} because their structural properties support exact and tractable computation of likelihoods, marginals, conditionals and other probabilistic queries \cite{PC20,ver21,peh15,she16,kho19}. Moreover, a PC represents more than the complete joint distribution: its internal nodes encode a hierarchy of tractable probability distributions over different variable scopes \cite{peh16,zha15}, providing multiple views of the learned distribution. 


In this work, we consider distributional OOD detection, where the objective is to determine whether a collection of observations is statistically consistent with the distribution represented by a trained PC. Generic distribution-comparison methods, such as Maximum Mean Discrepancy (MMD) \cite{gret12}, require reference samples from both the ID and test distributions and therefore operate as two-sample tests. Scalar likelihood \cite{ven23,nal19} and typicality-based \cite{nal19d} approaches may avoid an explicit two-sample statistic, but generally still require an ID reference or calibration sample, obtained either from held-out data or by sampling from the learned model. Moreover, relying only on the likelihood evaluated at the root maps every observation to a single scalar and therefore compares distributions only through the behavior of this statistic, while discarding the richer probabilistic information available throughout the circuit. This raises a central question: \textit{can the hierarchical probabilistic information encoded within a trained PC be used to construct a statistically calibrated distribution-level OOD detector without relying on held-out ID data?}

To address this question, we represent each observation using the likelihoods evaluated at a selected set of PC nodes, forming a Hierarchical Likelihood Vector (HLV) that preserves probabilistic information across multiple variable scopes. We then represent a distribution by the expectation of its HLV and define the Hierarchical Likelihood Distance (HLD) as the discrepancy between the expected HLVs of the learned and test distributions. Based on the resulting HLD test statistic, we construct a goodness-of-fit test for distributional OOD detection. Crucially, we show that the ID population mean and covariance quantities required by the test can be computed exactly and tractably from the trained PC, enabling an approximate closed-form decision threshold without requiring held-out ID data. Overall, our contributions are summarized as follows.

\begin{itemize}
    \item We introduce the Hierarchical Likelihood Vector (HLV), which represents an observation through the likelihoods evaluated at selected nodes of a PC and define the Hierarchical Likelihood Distance (HLD), a PC-induced pseudo-metric between distributions through the expectations of their HLVs. We further establish that HLD is a PC-induced Integral Probability Metric.

    \item We develop an HLD-based goodness-of-fit test for distributional OOD detection and show that the ID population mean and covariance quantities required by the test can be computed exactly and tractably from a structured-decomposable PC. This enables an analytic decision threshold without requiring held-out ID data.

    \item We evaluate the proposed framework on tabular and MNIST datasets, demonstrating improved distributional OOD detection over root-likelihood, uncertainty-, typicality- and kernel-based baselines. We further study the contribution of different node types, robustness to PC architecture, capacity and illustrate the HLV's ability to localize the variables associated with the shift.
\end{itemize}

\section{Background}

\subsection{Probabilistic Circuits}

A probabilistic circuit (PC) \(\mathcal C\) is a rooted directed acyclic graph that represents a joint probability distribution over a set of random variables \(\mathcal X=\{X_1,\ldots,X_d\}\). Each node \(n\) has a scope \(S_n\subseteq\mathcal X\), consisting of the variables appearing in the subcircuit rooted at \(n\) and represents a function \(p_n(\mathbf x_{S_n})\). A PC includes three node types: \textit{Input} nodes represent univariate distributions, \textit{sum} nodes represent convex mixtures of their children and \textit{product} nodes represent products of their children. Formally, each node \( n \) defines a distribution \( p_n \) recursively as:
\[
p_n(\mathbf x_{S_n}) = 
\begin{cases}
    f_n(\mathbf x_{S_n}), & \text{if } n \text{ is an input node} \\
    \prod_{c\in ch(n)} p_{c}(\mathbf x_{S_c}), & \text{if } n \text{ is a product node} \\
    \sum_{c\in ch(n)} \theta_{nc} \cdot p_c(\mathbf x_{S_c}), & \text{if } n \text{ is a sum node}
\end{cases}
\]
where \( f_n \) denotes a univariate input distribution, \( ch(n) \) represents the children of \( n \) and \( \theta_{nc} \) is the weight associated with the edge \( (n,c) \) and satisfy \( \theta_{nc} \in [0,1] \) for all \( c \in ch(n) \), \(
\sum_{c\in ch(n)} \theta_{nc}=1.\) A PC is evaluated bottom-up and the output of the root \(n_r \) gives the represented distribution 
\( p_C(\mathbf x)=p_{n_r}(\mathbf x). \)

A PC is \textit{smooth} if the children of every sum node have identical scopes and \textit{decomposable} if the children of every product node have pairwise disjoint scopes. It is \textit{structured-decomposable} if its product decompositions conform to a common v-tree over \(\mathcal X\). Under these conditions, every node represents a valid probability distribution over its scope, while likelihood evaluation and marginalization remain tractable. Moreover, products of compatible structured-decomposable PCs can be constructed tractably \cite{ver21}, a property that we later use to compute the moments of the proposed representation directly from the circuit.

\subsection{Related Work}
Likelihood-based scoring is a central paradigm for unsupervised OOD detection with generative models and has been studied extensively for variational autoencoders, normalizing flows and related deep generative architectures \cite{an15,nal19,kir20}. Its reliability, however, is limited by the likelihood paradox: OOD samples may receive higher likelihood than ID data \cite{nal19}. This observation has motivated methods that modify, calibrate or reinterpret likelihood-based evidence, including likelihood ratios \cite{ren19}, WAIC-based detection \cite{cho19}, input-complexity corrections \cite{ser20}, typicality-based detection \cite{nal19d} and likelihood regret \cite{xia20}. Despite their differences, these approaches base their decisions on one or a small number of scalar statistics derived from root-level likelihood evaluations and do not exploit the richer information encoded within the generative model.

Classical goodness-of-fit and two-sample tests based on empirical distribution functions are difficult to extend effectively to high-dimensional data, where their statistical power can deteriorate \cite{nal19d}. Kernel methods provide more flexible alternatives. Kernelized Stein discrepancy tests whether samples are consistent with a model using its score function \cite{chw16,liu16}, whereas MMD compares two distributions through their kernel mean embeddings \cite{gret12}. The former requires access to a score function, while the latter requires ID reference samples; both also depend on the choice of kernel. 

OOD detection with PCs has received comparatively limited attention. Tractable Dropout Inference (TDI) exploits the tractable structure of PCs to obtain predictive uncertainty estimates in a single forward pass \cite{ven23}, but is developed primarily for supervised prediction rather than unsupervised distributional testing. More closely related, Membership Circuits transform a trained PC into a per-sample membership test by combining node-level \(p\)-values bottom-up into a scalar value at the root \cite{wit26}. In contrast, our setting concerns batch-level OOD detection and the proposed representation retains the node-wise likelihood responses rather than collapsing the circuit hierarchy into a single scalar statistic.

\section{Methodology}

\subsection{Problem Formulation}
Let \(\mathcal C\) be a trained probabilistic circuit representing an in-distribution \(P\) over \(\mathcal X\). We restrict attention to discrete random variables with Bernoulli or categorical input distributions. As a result, the probability mass function represented at every node of \(\mathcal C\) takes values in \([0,1]\). This boundedness is used in the construction and statistical analysis of the proposed representation; extending the framework to continuous random variables is left as future work. 

At deployment, we observe a test batch \(
\mathcal D_Q=\{\mathbf x^{(t)}\}_{t=1}^{T}, 
\mathbf x^{(t)}\overset{\mathrm{i.i.d.}}{\sim} Q,
\) drawn from an unknown distribution \(Q\). We consider distributional OOD detection and test
\[
H_0:P=Q
\qquad\text{against}\qquad
H_1:P\neq Q.
\]
Rejecting \(H_0\) implies \(D_Q\) is distributionally OOD with respect to the distribution represented by \(\mathcal C\).

We assume that the trained circuit is the sole representation of the ID available at deployment; no held-out samples from \(P\) are available for constructing or calibrating the test.

\subsection{Hierarchical Likelihood Vector}
\begin{definition} [Hierarchical Likelihood Vector] Let \(\mathcal N={n_1,\ldots,n_m}\) be a set of selected nodes of the smooth and decomposable PC \(\mathcal C\). Each node \(n\in\mathcal N\) has scope \(S_n\) and represents a probability mass function \(p_n(\mathbf x_{S_n})\). For an observation \(\mathbf x\), define the node likelihood \(\ell_n(\mathbf x)=p_n(\mathbf x_{S_n}).\) The Hierarchical Likelihood Vector (HLV) is defined as:
\[\textbf{L}_\mathcal{C}(\mathbf x)=(\ell_{n_1}(\mathbf x),\ell_{n_2}(\mathbf x),\dots,\ell_{n_m}(\mathbf x))\in [0, 1]^{m}\]
 When the underlying PC is clear from the context, we simply write $\textbf{L}(\mathbf x)$.

\end{definition}
The HLV is constructed through a single bottom-up evaluation of the PC. Each coordinate in HLV measures the compatibility of the observation \(\mathbf x_{S_n}\) with the distribution represented by the corresponding subcircuit: input node provides variable-level summaries, while sum and product nodes provide summaries over progressively larger scopes. If the PC's root is included in \(\mathcal N\), its coordinate is the conventional likelihood \(p_{\mathcal C}(\mathbf x).\) Thus, the HLV can retain the root likelihood while augmenting it with probabilistically interpretable hierarchical summaries from within the circuit. The selected node set \(\mathcal N\) is fixed independently of the test batch.

\subsection{Hierarchical Likelihood Distance: Distribution Comparison using HLVs}
\begin{definition} [Hierarchical Likelihood Distance]
Let  \(\boldsymbol{\mu}_P=\mathbb{E}_{\mathbf X\sim P}[\textbf{L(X)}]\) and  \(\boldsymbol{\mu}_Q=\mathbb{E}_{\mathbf X\sim Q}[\textbf{L(X)}]\) denote the expected HLVs of the distributions $P$ and $Q$, respectively. The Hierarchical Likelihood Distance (HLD) between $P$ and $Q$ is defined as:
\[d_{\mathcal C}(P,Q)=\|\boldsymbol{\mu}_P-\boldsymbol{\mu}_Q\|_2\]
 \end{definition}
The HLV and HLD are defined for smooth and decomposable PCs. Exact
computation of their population moments additionally requires
\(\mathcal C\) to be structured decomposable. The HLD measures the discrepancy between two probability distributions through the Euclidean distance between their expected HLV representations. It is non-negative, symmetric and satisfies the triangle inequality. However, it is generally a pseudo-metric, as distinct probability distributions may induce identical expected HLVs. It is easy to see that if $P=Q$, then $d_{\mathcal C}(P,Q)=0$ and so if $d_{\mathcal C}(P, Q) > 0$, then $P \neq Q$.

The converse is not necessarily true, \(d_{\mathcal C}(P, Q)=0\) does not necessarily imply \(P=Q\). The discriminative power of HLD is determined by the finite family of node likelihood functions induced by the selected nodes of the probabilistic circuit. Equivalently, HLD compares distributions through their projections onto the function space spanned by these node likelihoods. Thus, HLD cannot distinguish two distributions that agree on the expectations of all selected node-likelihood functions.

\subsection{HLD as an Integral Probability Metric}
HLD admits an equivalent interpretation as an Integral Probability Metric (IPM), placing it within a well established family of distributional discrepancy measures.
\begin{theorem}Define the function class:
\[\mathcal{F}_\mathcal{C}=\{f_\textbf{a}:\mathbf X\rightarrow \textbf{a}^TL_\mathcal{C}(\mathbf X):\textbf{a}\in \mathbb{R}^m,\|\textbf{a}\|_2\leq1\}\]
Then, \(d_{\mathcal C}(P,Q)=\mathrm{IPM}_{\mathcal{F_\mathcal{C}}}(P,Q)\)
\end{theorem}
\begin{proof}
By definition,
\begin{align*}
\mathrm{IPM}_{\mathcal{F}_\mathcal{C}}(P,Q)
&=\sup_{\|\textbf{a}\|_2\le1}
\left|\textbf{a}^\top(\boldsymbol{\mu}_P-\boldsymbol{\mu}_Q)\right|\\
&=\|\boldsymbol{\mu}_P-\boldsymbol{\mu}_Q\|_2
=d_{\mathcal C}(P,Q),
\end{align*}
When \(\boldsymbol{\mu}_P=\boldsymbol{\mu}_Q\), both sides are zero; otherwise, the supremum is attained at \( \textbf{a}=\frac{\boldsymbol{\mu}_P-\boldsymbol{\mu}_Q}{\|\boldsymbol{\mu}_P-\boldsymbol{\mu}_Q\|_2}\), by the dual characterization of the Euclidean norm.
\end{proof}
This theorem has an important consequence: the proposed detector is not an ad hoc feature comparison. Rather, it is an IPM over a function class induced directly by the PC.  Unlike generic representation-based IPMs, whose function classes are learned implicitly, every function in $\mathcal{F}_C$ is a linear combination of probabilistically meaningful queries induced by the PC, giving the entire function class a direct probabilistic interpretation. Thus, as noted earlier, the discriminative power of the HLD is governed entirely by the probabilistic query family induced by the PC. Through this mechanism HLD can be connected to other divergence measures as well that is discussed in the Supplementary Material, Section~A.

\subsection{Goodness-of-Fit Test using HLD}
Given $D_Q$ from an unknown distribution $Q$, we estimate the expected HLV as
\(
\boldsymbol{\hat{\mu}}_Q=\frac{1}{T}\sum_{t=1}^{T} L(\mathbf x^{(t)}).
\)
The HLD test statistic between the learned distribution $P$ and the test distribution $Q$ is then
\[
\hat{d}_{\mathcal C}(P,Q)=\|\boldsymbol{\mu}_P-\boldsymbol{\hat{\mu}}_Q\|_2,
\]
where $\boldsymbol{\mu}_P=\mathbb{E}_{\mathbf X\sim P}[\textbf{L(X)}]$ denotes the expected HLV under the distribution represented by $\mathcal C$.

Under the null hypothesis \(H_0:P=Q\), the HLD test statistic is
generally non-zero because \(\widehat{\boldsymbol{\mu}}_Q\) is
estimated from a finite test batch. Let
\[
\mathbf Z_T
=
\widehat{\boldsymbol{\mu}}_Q-\boldsymbol{\mu}_P,
\qquad
\Delta_T
=
\widehat d_{\mathcal C}(P,Q)
=
\|\mathbf Z_T\|_2,
\]
and let
\[
\boldsymbol{\Sigma}_P
=
\operatorname{Cov}_{\mathbf X\sim P}
\left[\mathbf L_{\mathcal C}(\mathbf X)\right].
\]
Under \(H_0\),
\[
\mathbb E[\mathbf Z_T]=\mathbf 0,
\qquad
\operatorname{Cov}(\mathbf Z_T)
=
\frac{1}{T}\boldsymbol{\Sigma}_P.
\]
As the HLV is bounded, the multivariate central limit theorem and
standard results for Gaussian quadratic forms \cite{van98,imh61} give 
\[
\sqrt{T}\,\mathbf Z_T
\xrightarrow{d}
\mathcal N(\mathbf 0,\boldsymbol{\Sigma}_P),
\qquad
T\Delta_T^2
\xrightarrow{d}
\sum_{j=1}^{m}\lambda_j\chi_j^2(1),
\]
where \(\lambda_1,\ldots,\lambda_m\) are the eigenvalues of
\(\boldsymbol{\Sigma}_P\) and the \(\chi_j^2(1)\) variables are
independent. Thus, the asymptotic null distribution of
\(\Delta_T^2\) is a scaled generalized chi-square distribution. Its
mean is exactly
\[
\mathbb E[\Delta_T^2]
=
\frac{1}{T}\operatorname{tr}(\boldsymbol{\Sigma}_P),
\]
and its leading-order asymptotic variance is
\[
\operatorname{Var}(\Delta_T^2)
\approx
\frac{2}{T^2}\operatorname{tr}(\boldsymbol{\Sigma}_P^2).
\]

For a prescribed significance level \(\alpha\), \(H_0\) is rejected
when \(\Delta_T\) exceeds the \((1-\alpha)\)-quantile of its null
distribution. Although this quantile depends on the complete
eigenvalue spectrum of \(\boldsymbol{\Sigma}_P\), an approximate
closed-form threshold can be obtained by matching the first two
moments of the generalized chi-square distribution:
\[
\Delta_T^2
\approx
\mathcal N\left(
\frac{\operatorname{tr}(\boldsymbol{\Sigma}_P)}{T},
\frac{2\operatorname{tr}(\boldsymbol{\Sigma}_P^2)}{T^2}
\right).
\]
Thus, the closed-form threshold can be computed as,
\[
\tau_\alpha
\approx
\sqrt{
\frac{\operatorname{tr}(\boldsymbol{\Sigma}_P)}{T}
+
\frac{z_{1-\alpha}}{T}
\sqrt{2\operatorname{tr}(\boldsymbol{\Sigma}_P^2)}
},
\]
and the test rejects \(H_0\) whenever
\(\Delta_T>\tau_\alpha\).

This approximation is expected to be accurate when no small number of eigenvalues of \(\boldsymbol{\Sigma}_P\) dominates its spectrum\cite{bil95}; otherwise, the quantile of the generalized chi-square distribution may be computed numerically from the eigenvalues. The threshold depends on the covariance of the HLV under \(P\) only through \(\tr(\boldsymbol{\Sigma}_P)\) and \(\tr(\boldsymbol{\Sigma}_P^2)\).

The central challenge is therefore no longer the hypothesis test itself, but the computation of the quantities \(\boldsymbol{\mu}_P, \tr(\boldsymbol{\Sigma}_P),\) and \(\tr(\boldsymbol{\Sigma}_P^2)\). Classical multivariate hypothesis tests estimate these statistics from an independent ID calibration set. In contrast, we show that the structural properties of PCs enable these quantities to be exactly and tractably computed from the learned PC, eliminating the need for additional ID samples.

We first establish the following structural property of structured-decomposable PCs, which underlies the exact computations in this section.
\begin{lemma}[Scope Trichotomy]
Let \(\mathcal C\) be a structured-decomposable and smooth PC with respect to a v-tree \(\mathcal T\). For any two nodes \(n_i,n_j \in \mathcal C\) with scopes \(S_i, S_j\) exactly one of the following holds: \(S_i\cap S_j =\emptyset \) or \(S_i=S_j\) or one of \(S_i, S_j\) strictly contains the other.    
\label{lemma:scope}
\end{lemma}
\begin{proof}
    Refer to the Supplementary Material, Section~B.
\end{proof}
\subsubsection{Exact Computation of \(\boldsymbol{\mu}_P\)}
We show that for smooth and structured-decomposable PCs, every component of \(\boldsymbol{\mu}_P\) can be computed exactly and tractably from the trained PC. For a selected node $n$ with scope $S_n$, let $\ell_n(\mathbf X)=p_n(\mathbf{X}_{S_n})$ denote the corresponding entry of the HLV, where $p_n$ is the probability distribution represented by the subcircuit rooted at $n$. The corresponding component of the population mean is
\[
\mu_{P_n}
=
\mathbb{E}_{\mathbf X\sim P}\!\left[\ell_n(\mathbf X)\right]
=
\sum_{\mathbf x}
p(\mathbf x)\,p_n(\mathbf x_{S_n})\,.
\]
As $p_n(\mathbf{x}_{S_n})$ depends only on the variables in $S_n$, the variables outside the scope can be marginalized out, yielding
\[
\mu_{p_n}
=
\sum_{\mathbf x_{S_n}}
p_{S_n}(\mathbf{x}_{S_n})\,p_n(\mathbf{x}_{S_n})\,
\]
where $p_{S_n}$ denotes the marginal distribution of the root PC over the scope $S_n$. Thus, computing $\mu_{P_n}$ reduces to evaluating the overlap between the node distribution $p_n$ and the corresponding marginal of the root distribution.

The above expectation can be evaluated exactly using standard PC operations. The marginal distribution $p_{S_n}$ is first obtained by marginalizing the root PC over the variables outside $S_n$. As the PC is smooth and structured-decomposable, the marginalized root distribution $p_{S_n}$ and the subcircuit distribution $p_n$ are both represented over the same induced v-tree on the scope $S_n$. Consequently, their product can be computed exactly using the tractable product operation for PCs \cite{ver21}. Furthermore, the resulting product PC is also smooth and decomposable. Marginalizing this product circuit over all variables in $S_n$ yields the exact value of $\mu_{Pn}$. Repeating this procedure for every selected node computes the complete population mean vector $\boldsymbol{\mu}_P$. 

As both \(p_{S_n}\) and \(p_n\) contain at most \(N_r\) nodes, their decomposable product contains at most \(O(N_r^2)\) nodes in the worst case. The subsequent marginalization over the product circuit is linear in its size, yielding an overall worst-case complexity of $O(N_r^2)$ for computing a single component $\boldsymbol{\mu}_{Pn}$. Therefore, if the HLV is constructed using $S$ selected nodes, the complete population mean vector $\boldsymbol{\mu}_P$ is computed exactly in $O(SN_r^2)$ time. If the HLV uses all nodes of the PC, then the complexity becomes $O(N_r^3)$.
\subsubsection{Exact computation of $\boldsymbol{\Sigma}_P$}
The second and third quantities required by the hypothesis test are the traces $\tr(\boldsymbol{\Sigma}_P)$ and $\tr(\boldsymbol{\Sigma}_P^2)$. We show that the full covariance matrix $\boldsymbol{\Sigma}_P$ can be computed exactly. The $ij^{th}$ entry in this matrix is of the form  \(\mathbb{E}_{\mathbf X\sim P}[\ell_i(\mathbf X)\ell_j(\mathbf X)]-\mu_{Pi}\mu_{Pj},\) so the computation reduces to evaluating the second order moment
\[\mathbb{E}_{\mathbf X\sim P}\!\left[\ell_i(\mathbf X)\,\ell_j(\mathbf X)\right]=\sum_{\mathbf x}p(\mathbf x)\,p_i(\mathbf{x}_{S_i})\,p_j(\mathbf{x}_{S_j}).\]
As before, $p_i$ and $p_j$ depend only on the variables in $S_i$ and $S_j$, the variables outside $S_i\cup S_j$ are marginalized out. Let \(S_{ij}\) denote the smallest scope in the induced v-tree
containing \(S_i\cup S_j\), i.e., the scope of the lowest common
ancestor of the v-tree nodes associated with \(S_i\) and \(S_j\), yielding \[\mathbb{E}_{\mathbf X\sim P}[\ell_i(\mathbf X)\ell_j(\mathbf X)]=\sum_{\mathbf x_{S_{ij}}}p_{S_{ij}}(\mathbf{x}_{S_{ij}})\,p_i(\mathbf{x}_{S_i})\,p_j(\mathbf{x}_{S_j}).\]
This expectation is evaluated using the same sequence of tractable PC operations described in the previous subsection. The root PC is first marginalized onto \(S_{ij}\), after which the product
\( p_{S_{ij}} \otimes p_i \otimes p_j \)
is constructed. The tractable product operation applies to any pair of compatible PCs \cite{ver21}. In our setting, the operands are obtained from the same structured-decomposable PC and therefore conform to a common v-tree. By Lemma \ref{lemma:scope}, the scopes \(S_i\) and \(S_j\) are either identical, nested, or disjoint; partial overlap cannot occur. In each of these cases, the corresponding subcircuits are compatible and the product operation can be applied successively to construct \(p_{S_{ij}} \otimes p_i \otimes p_j\). Marginalizing the resulting product circuit over \(S_{ij}\) yields the exact value of
\(
\mathbb E_{\mathbf X\sim P}
\left[\ell_i(\mathbf X)\ell_j(\mathbf X)\right]
\)
and hence \((\boldsymbol{\Sigma}_P)_{ij}\). Repeating this construction for all pairs of selected nodes gives the full covariance matrix \(\boldsymbol{\Sigma}_P\).

As each second-order moment is computed in worst-case time $O(N_r^3)$, the overall complexity of computing $\boldsymbol{\Sigma}_P$ is $O(S^2N_r^3)$. \textit{NOTE: As computing $\boldsymbol{\mu}_P$, $\boldsymbol{\Sigma}_P$ is performed only once after training, it does not affect the online complexity of OOD detection.}

\section{Experiments}
\subsection{Experimental Setup}
\paragraph{Datasets.}
We evaluate HLD on the Adult, Covertype, Sensorless, Census-KDD, Connect-4 tabular datasets from the UCI Machine Learning Repository \cite{bec96,bla98,tro95,cen00,bat13} and on the binarized MNIST image dataset \cite{lec98}. We adopt a class-wise OOD protocol. For each experiment, one class is treated as ID, while each remaining class is considered as a separate OOD. This produces multiple ID--OOD pairs for each dataset. Dataset preprocessing and the resulting pairs are described in Supplementary Material, Section~C.

\paragraph{Evaluation Protocol.} 
The experiments follow the problem formulation introduced in methodology, where the trained PC is assumed to be the sole representation of ID. HLD computes all population quantities required by the hypothesis test analytically from the trained PC. In contrast, baselines requiring ID reference samples obtain them by sampling from the trained PC. Unless otherwise stated, the reference set size is chosen to equal the test batch size \cite{efr79,efr93}. 

For each ID--OOD pair, a single PC is trained using only the training data from the designated ID class. For each test-batch size \(T\), we perform 500
independent Monte Carlo trials. For HLD, the population quantities \(\boldsymbol{\mu}_P\), \(\boldsymbol{\Sigma}_P\) and hence the decision threshold, are computed directly from the resulting PC. Thus, HLD does not require an ID reference sample. For baselines that require an ID reference set, we independently draw \(T\) samples from the trained PC, matching the size of the test batch. A fresh test batch is also drawn in each trial from either the ID or OOD evaluation data, depending on whether Type-I error or detection power is being measured. Regenerating the reference and test samples in every trial account for  sampling variability.

\paragraph{Density Models.}
Unless otherwise stated, we use Hidden Chow--Liu Trees (HCLTs)
\cite{liu21}, implemented in PyJuice \cite{liu24}, with four latent
states. The resulting PCs are smooth and structured-decomposable. We use categorical input distributions for the tabular datasets and Bernoulli input distributions for binarized MNIST. Model training and convergence criteria are described in Supplementary Material, Section~C.  Detection stays near-perfect across HCLT capacities and RAT-SPN, with HCLT competitive or better at matched budget. Further details are provided in the Supplementary Material, Section~D.

\paragraph{Baselines.}
We compare HLD with representative likelihood-based, uncertainty-based and distributional OOD detectors: Root Likelihood (RootLL), Typicality, Tractable Dropout Inference (TDI) and Maximum Mean Discrepancy (MMD). Each baseline is adapted to the same batch-level OOD setting. Further details of the baseline methods are provided in the Supplementary Material, Section~C.

\paragraph{Performance Measures.}
We evaluate each detector at a prescribed significance level
\(\alpha\) using its false-positive rate and detection power. We
report two false-positive rates. First, \(\mathrm{FPR}_{\mathrm{model}}\) is measured on test batches sampled from the trained PC and evaluates calibration under the model-based null hypothesis \(Q=P\). Second, \(\mathrm{FPR}_{\mathrm{data}}\) is measured on batches drawn from the held-out ID test set and evaluates false alarms relative to the underlying data distribution. The difference between these quantities reflects model misspecification: \(\mathrm{FPR}_{\mathrm{model}}\) is the empirical Type-I error of the statistical test, whereas \(\mathrm{FPR}_{\mathrm{data}}\) also depends on how accurately the PC represents the true ID distribution. Detection power is the proportion of OOD test batches for which the null hypothesis is rejected. Unless otherwise stated, results are averaged over all ID--OOD pairs and 500 Monte Carlo trials.

\subsection{Results}

We organize the empirical evaluation around three research questions that analyze the HLD framework:
\begin{itemize}
    \item \textbf{Q1:} Do hierarchical likelihood summaries, captured in the HLV, improve distributional OOD detection over root-level and existing baselines?
    \item \textbf{Q2:} How does the choice of PC node types affect the detection performance of HLD?
    \item \textbf{Q3:} Can node-level HLV discrepancies localize the variables associated with a distribution shift?
\end{itemize}
\subsubsection{\textbf{Q1:} Do hierarchical likelihood summaries, captured in the HLV, improve distributional OOD detection over root-level and existing baselines?}
\paragraph{Tabular datasets.}
Table~\ref{tab:tabular-19pairs} reports results aggregated over 500
Monte Carlo trials for every ID--OOD pair across the five tabular
datasets at $\alpha=0.05$. Results for each dataset at other $\alpha$ is presented in Supplementary Material, Section E and D respectively.

Under the model-based null, all methods remain close to the target
significance level \(\alpha=0.05\). HLD is slightly liberal, with
\(\mathrm{FPR}_{\mathrm{model}}\) between \(0.059\) and \(0.067\),
while the remaining methods generally remain closer to \(0.05\).

On test batches drawn from held-out ID data, false-positive rates
increase with batch size because discrepancies between the learned
PC and the underlying data distribution (if any) become easier to detect. Although HLD does not attain the lowest
\(\mathrm{FPR}_{\mathrm{data}}\) at every small batch size, its
false-positive rate increases substantially slower for larger
batches. At \(T=1000\), HLD obtains an
\(\mathrm{FPR}_{\mathrm{data}}\) of \(0.178\), compared with \(0.807\)
for MMD, \(0.446\) for RootLL, \(0.753\) for Typicality and \(0.428\) for TDI. \textit{This indicates that HLD is comparatively more robust to model misspecification as the test batch grows.}

HLD also achieves the highest detection power for
\(10\leq T\leq100\), reaching \(0.716\), \(0.932\) and \(0.987\) at
\(T=10\), \(50\) and \(100\) respectively. MMD is the closest
competitor, with corresponding powers of \(0.681\), \(0.892\),
\(0.979\) and both methods reach essentially perfect detection from
\(T=200\). In contrast, the scalar-score baselines improve more
slowly: at \(T=1000\), RootLL, Typicality and TDI attain powers of
\(0.774\), \(0.851\) and \(0.785\) respectively. 

The variability across datasets and ID--OOD pairs provides an
additional distinction between the methods. For
\(\mathrm{FPR}_{\mathrm{data}}\), HLD has the smallest or nearly
smallest standard deviation across the evaluated batch sizes. The
difference becomes particularly pronounced for moderate and large
batches; at \(T=500\), its standard deviation is \(0.06\), compared with \(0.35\), \(0.12\), \(0.31\) and \(0.12\) for MMD, RootLL, Typicality and TDI respectively. Thus, HLD's robustness to model misspecification is more consistent across the evaluated datasets.

A similar pattern appears in detection power. HLD exhibits
substantially lower variability at small and moderate batch sizes. At \(T=100\), its standard deviation is \(0.03\), compared with \(0.05\) for MMD and approximately \(0.45\) for the scalar-score baselines. MMD becomes comparably stable once its detection power approaches saturation, whereas RootLL, Typicality and TDI retain large variability across ID--OOD pairs. This indicates that the aggregate performance of the scalar baselines conceals substantial differences in their ability to detect particular distribution shifts. 

\begin{table}[t]
\centering
\small
\setlength{\tabcolsep}{1mm}
\begin{tabular}{rccccc}
\toprule
$T$ & HLD & MMD & RootLL & Typicality & TDI \\
\midrule
\multicolumn{6}{l}{\emph{(a) \(\mathrm{FPR}_{\mathrm model}\) --- PC-vs-PC null (target $0.05$)}} \\
1    & $.059{\pm}.04$ & $\mathbf{.029{\pm}.02}$ & $\textit{.049}{\pm}\textit{.01}$ & $.053{\pm}.01$ & $\textit{.049}{\pm}\textit{.01}$ \\
10   & $.063{\pm}.01$ & $\mathbf{.049{\pm}.01}$ & $.057{\pm}.01$ & $\textit{.056}{\pm}\textit{.01}$ & $.057{\pm}.01$ \\
50   & $.066{\pm}.01$ & $\textit{.052}{\pm}\textit{.01}$ & $.055{\pm}.01$ & $\mathbf{.050{\pm}.02}$ & $.053{\pm}.01$ \\
100  & $.066{\pm}.02$ & $.051{\pm}.01$ & $\textit{.047}{\pm}\textit{.01}$ & $.048{\pm}.01$ & $\mathbf{.046{\pm}.01}$ \\
200  & $.065{\pm}.01$ & $.051{\pm}.01$ & $\mathbf{.044{\pm}.01}$ & $.047{\pm}.01$ & $\textit{.045}{\pm}\textit{.01}$ \\
500  & $.067{\pm}.01$ & $\textit{.048}{\pm}\textit{.02}$ & $\textit{.048}{\pm}\textit{.02}$ & $\mathbf{.047{\pm}.01}$ & $.049{\pm}.01$ \\
1000 & $.061{\pm}.01$ & $\mathbf{.047{\pm}.01}$ & $.051{\pm}.01$ & $\textit{.050}{\pm}\textit{.01}$ & $.051{\pm}.01$ \\
\midrule
\multicolumn{6}{l}{\emph{(b) \(\mathrm{FPR}_{\mathrm{data}}\) --- PC reference -vs-\ real held-out ID}} \\
1    & $\textit{.037}{\pm}{.02}$ & $\mathbf{.023{\pm}.02}$ & $.047{\pm}.02$ & $.050{\pm}.04$ & $.048{\pm}.02$ \\
10   & $.065{\pm}.02$ & $.071{\pm}.02$ & $\textit{.054}{\pm}\textit{.03}$ & $\mathbf{.048{\pm}.04}$ & $.055{\pm}.03$ \\
50   & $.076{\pm}.02$ & $.096{\pm}.04$ & $\mathbf{.060{\pm}.03}$ & $\textit{.067}{\pm}\textit{.06}$ & $\mathbf{.060{\pm}.03}$ \\
100  & $\textit{.086}{\pm}\textit{.02}$ & $.166{\pm}.10$ & $\mathbf{.070{\pm}.03}$ & $.091{\pm}.09$ & $\mathbf{.070{\pm}.03}$ \\
200  & $\mathbf{.103{\pm}.03}$ & $.362{\pm}.24$ & $.107{\pm}.04$ & $.139{\pm}.12$ & $\textit{.104}{\pm}\textit{.04}$ \\
500  & $\mathbf{.135{\pm}.06}$ & $.678{\pm}.35$ & $.227{\pm}.12$ & $.420{\pm}.31$ & $\textit{.219}{\pm}\textit{.12}$ \\
1000 & $\mathbf{.178{\pm}.08}$ & $.807{\pm}.23$ & $.446{\pm}.23$ & $.753{\pm}.22$ & $\textit{.428}{\pm}\textit{.23}$ \\
\midrule
\multicolumn{6}{l}{\emph{(c) Detection rate --- PC reference -vs-\ OOD}} \\
1    & $.040{\pm}.05$ & $.028{\pm}.04$ & $\textit{.090}{\pm}\textit{.10}$ & $\mathbf{.119{\pm}.15}$ & $.089{\pm}.10$ \\
10   & $\mathbf{.716{\pm}.07}$ & $\textit{.681}{\pm}\textit{.38}$ & $.276{\pm}.33$ & $.352{\pm}.37$ & $.275{\pm}.33$ \\
50   & $\mathbf{.932{\pm}.11}$ & $\textit{.892}{\pm}\textit{.22}$ & $.461{\pm}.44$ & $.555{\pm}.46$ & $.463{\pm}.44$ \\
100  & $\mathbf{.987{\pm}.03}$ & $\textit{.979}{\pm}\textit{.05}$ & $.542{\pm}.45$ & $.608{\pm}.47$ & $.544{\pm}.45$ \\
200  & $\mathbf{1.00{\pm}.00}$ & $\mathbf{1.00{\pm}.00}$ & $.619{\pm}.45$ & $\textit{.654}{\pm}\textit{.46}$ & $.622{\pm}.44$ \\
500  & $\mathbf{1.00{\pm}.00}$ & $\mathbf{1.00{\pm}.00}$ & $.696{\pm}.47$ & $\textit{.773}{\pm}\textit{.38}$ & $.701{\pm}.41$ \\
1000 & $\mathbf{1.00{\pm}.00}$ & $\mathbf{1.00{\pm}.00}$ & $.774{\pm}.37$ & $\textit{.851}{\pm}\textit{.35}$ & $.785{\pm}.35$ \\
\bottomrule
\end{tabular}
\caption{Detection results averaged across the tabular datasets with $\alpha=0.05$ (\textbf{Bold} is the best performing method and \textit{Italics} is the second best).}
\label{tab:tabular-19pairs}
\end{table}
\begin{table}[h]
\centering
\small
\setlength{\tabcolsep}{1mm}
\begin{tabular}{rccccc}
\toprule
$T$ & HLD & MMD & RootLL & Typicality & TDI \\
\midrule
\multicolumn{6}{l}{\emph{(a)  \(\mathrm{FPR}_{\mathrm model}\) --- PC-vs-PC null (target $0.05$)}} \\
1    & $.089{\pm}.03$ & $\mathbf{.038{\pm}.01}$ & $\textit{.049}{\pm}\textit{.01}$ & $\textit{.049}{\pm}\textit{.01}$ & $\textit{.049}{\pm}\textit{.01}$ \\
10   & $.065{\pm}.01$ & $\mathbf{.053{\pm}.01}$ & $\textit{.054}{\pm}\textit{.02}$ & $\textit{.054}{\pm}\textit{.01}$ & $\mathbf{.053{\pm}.01}$ \\
50   & $.063{\pm}.01$ & $.051{\pm}.01$ & $.051{\pm}.01$ & $\mathbf{.048{\pm}.01}$ & $\textit{.050}{\pm}\textit{.01}$ \\
100  & $.063{\pm}.01$ & $\textit{.048}{\pm}\textit{.01}$ & $\mathbf{.046{\pm}.01}$ & $.052{\pm}.01$ & $\mathbf{.046{\pm}.01}$ \\
200  & $.065{\pm}.01$ & $\textit{.050}{\pm}\textit{.01}$ & $.048{\pm}.01$ & $.051{\pm}.01$ & $\mathbf{.047{\pm}.01}$ \\
500  & $.067{\pm}.01$ & $\mathbf{.046{\pm}.01}$ & $\textit{.050}{\pm}\textit{.01}$ & $.050{\pm}.02$ & $.051{\pm}.01$ \\
1000 & $.065{\pm}.02$ & $\mathbf{.047{\pm}.01}$ & $.051{\pm}.01$ & $\textit{.048}{\pm}\textit{.01}$ & $.052{\pm}.01$ \\
\midrule
\multicolumn{6}{l}{\emph{(b) \(\mathrm{FPR}_{\mathrm{data}}\) --- PC reference -vs-\ real held-out ID}} \\
1    & $.099{\pm}.04$ & $\mathbf{.045{\pm}.01}$ & $.079{\pm}.02$ & $.085{\pm}.03$ & $\textit{.079}{\pm}\textit{.01}$ \\
10   & $\textit{.078}{\pm}\textit{.01}$ & $\mathbf{.062{\pm}.01}$ & $.086{\pm}.02$ & $.122{\pm}.03$ & $.085{\pm}.02$ \\
50   & $\textit{.079}{\pm}\textit{.02}$ & $\mathbf{.068{\pm}.01}$ & $.097{\pm}.02$ & $.141{\pm}.03$ & $.093{\pm}.02$ \\
100  & $\textit{.081}{\pm}\textit{.02}$ & $\mathbf{.080{\pm}.02}$ & $.105{\pm}.02$ & $.181{\pm}.04$ & $.099{\pm}.02$ \\
200  & $\mathbf{.093{\pm}.03}$ & $\textit{.115}{\pm}\textit{.03}$ & $.136{\pm}.05$ & $.226{\pm}.07$ & $.123{\pm}.04$ \\
500  & $\mathbf{.123{\pm}.06}$ & $.290{\pm}.10$ & $.224{\pm}.10$ & $.376{\pm}.17$ & $\textit{.197}{\pm}\textit{.09}$ \\
1000 & $\mathbf{.178{\pm}.12}$ & $.592{\pm}.15$ & $.478{\pm}.19$ & $.574{\pm}.27$ & $\textit{.429}{\pm}\textit{.17}$ \\
\midrule
\multicolumn{6}{l}{\emph{(c) Detection rate --- PC reference -vs-\ OOD}} \\
1    & $.419{\pm}.28$ & $.193{\pm}.20$ & $.498{\pm}.27$ & $\mathbf{.624{\pm}.27}$ & $\textit{.500}{\pm}\textit{.27}$ \\
10   & $\mathbf{.997{\pm}.01}$ & $\textit{.962}{\pm}\textit{.09}$ & $.905{\pm}.20$ & $.944{\pm}.18$ & $.906{\pm}.20$ \\
50   & $\mathbf{1.00{\pm}.00}$ & $\mathbf{1.00{\pm}.00}$ & $.969{\pm}.15$ & $\textit{.977}{\pm}\textit{.14}$ & $.969{\pm}.15$ \\
100  & $\mathbf{1.00{\pm}.00}$ & $\mathbf{1.00{\pm}.00}$ & $.977{\pm}.14$ & $\textit{.980}{\pm}\textit{.13}$ & $.977{\pm}.14$ \\
200  & $\mathbf{1.00{\pm}.00}$ & $\mathbf{1.00{\pm}.00}$ & $.980{\pm}.13$ & $\textit{.981}{\pm}\textit{.13}$ & $.980{\pm}.13$ \\
500  & $\mathbf{1.00{\pm}.00}$ & $\mathbf{1.00{\pm}.00}$ & $.981{\pm}.13$ & $\textit{.984}{\pm}\textit{.11}$ & $.982{\pm}.12$ \\
1000 & $\mathbf{1.00{\pm}.00}$ & $\mathbf{1.00{\pm}.00}$ & $.983{\pm}.11$ & $\textit{.987}{\pm}\textit{.09}$ & $.984{\pm}.11$ \\
\bottomrule
\end{tabular}
\caption{Detection results on MNIST (7$\times$7 resolution) with $\alpha=0.05$ (\textbf{Bold} is the best performing method and \textit{Italics} is the second best).}
\label{tab:mnist-90pairs}
\end{table}

\paragraph{MNIST.}
Table \ref{tab:mnist-90pairs} reports binarized MNIST at \(7\times7\) resolution, the advantage of HLD
again emerges when multiple test observations are available. Although HLD does not perform best at \(T=1\), its detection power reaches \(0.997\) at \(T=10\), compared with \(0.962\) for
MMD, \(0.905\) for RootLL, \(0.944\) for Typicality and \(0.906\) for TDI. At \(T=100\), HLD and MMD attain perfect detection, while the other baselines remain slightly lower and exhibit greater variability. HLD also becomes increasingly robust on held-out ID data as the batch size grows relative to the competing methods.

The full-resolution \(28\times28\) results, reported in Supplementary Material, Section~E, reveal the effect of model misspecification. All methods achieve perfect or near-perfect detection at relatively small batch sizes, but the \(\mathrm{FPR}_{\mathrm{data}}\) of HLD at \(T=1000\) increases from \(0.178\) at \(7\times7\) resolution to \(0.823\) at \(28\times28\). This behavior is consistent with the greater difficulty of accurately modeling the higher-dimensional pixel distribution. Because HLD aggregates discrepancies across multiple node-likelihood coordinates, small systematic modeling errors accumulate and become increasingly detectable as \(T\) grows. Thus, while the test is calibrated with respect to the distribution represented by the trained PC, its false-positive rate on real ID data necessarily depends on how faithfully that PC 
represents the underlying data distribution.

\subsection{Q2: How does the choice of PC node types affect the
detection performance of HLD?}

Table~\ref{tab:node-ablation} compares HLVs constructed from different combinations of leaf, sum and product nodes, with detection power aggregated across the five tabular datasets as \(T\) increases. We report results only up to \(T=100\), as detection saturates at larger batch sizes. The combination of leaf and sum nodes performs best at every evaluated batch size. Adding product nodes to the leaf+sum representation reduces detection power while increasing the HLV dimension.

Although product nodes may add interaction moments to the HLV
representation, these additional coordinates do not translate into
improved detection power in the evaluated settings. In particular,
leaf+sum consistently outperforms leaf+sum+product despite having a
substantially lower dimension. This result shows that the product-node coordinates provide no empirical benefit for the circuits and
distribution shifts considered here. Determining whether this  behavior arises from redundancy, coordinate scaling or the covariance structure of the augmented HLV requires a more detailed analysis and is left for future work.

Dataset-specific results in Supplementary Material, Section~D, show that neither leaf nor sum nodes dominate uniformly: leaf nodes carry more OOD signal on some datasets while sum nodes are more informative on others. Combining these two node families therefore provides the most robust choice across datasets and we use leaf and sum nodes to construct the HLV in the remaining experiments.

\begin{table}\centering\small
\setlength{\tabcolsep}{5pt}
\begin{tabular}{lrrrrr}
\toprule
node set & $\dim$ & $T{=}1$ & $T{=}10$ & $T{=}50$ & $T{=}100$  \\
\midrule
leaf$+$sum          & 118 & \textbf{0.050} & \textbf{0.369} & \textbf{0.844} & \textbf{0.918} \\
leaf$+$prod         & 191 & 0.036 & 0.149 & 0.615 & 0.805  \\
sum$+$prod          & 143 & 0.038 & 0.092 & 0.455 & 0.699 \\
leaf$+$sum$+$prod   & 261 & \textit{0.047} & \textit{0.177} & \textit{0.622} & \textit{0.807}  \\
\bottomrule
\end{tabular}
\caption{Detection rate averaged across the Tabular datasets on different node combinations (\textbf{Bold} is the best performing method and \textit{Italics} is the second best). }
\label{tab:node-ablation}
\end{table}

\subsection{Q3: Can node-level HLV discrepancies localize the
variables associated with a distribution shift?}

Unlike scalar-score detectors, HLD retains the discrepancy associated with every selected PC node. In particular, the squared test statistic decomposes as
\(
\Delta_T^2
=
\left\|
\widehat{\boldsymbol{\mu}}_Q-\boldsymbol{\mu}_P
\right\|_2^2
=
\sum_{j=1}^{m}
\left(
\widehat{\mu}_{Qj}-\mu_{Pj}
\right)^2.
\)
We therefore define the contribution of node \(n_j\) as
\(
\delta_j
=
\left|
\widehat{\mu}_{Qj}-\mu_{Pj}
\right|,
\)
rank the selected nodes by \(\delta_j\) and inspect the scopes of the top-\(K\) nodes. As the scope of each node identifies the input variables represented by its subcircuit, this procedure associates the detected discrepancy with particular variable subsets. These attributions should be interpreted as localization cues rather than causal explanations as the node scopes may be overlapping or nested.

Figure~\ref{fig:comparison} illustrates this procedure for a PC trained on digits \(0\),\(7\),\(6\) with test batches containing digits \(6\),\(8\),\(9\). For the \(0\)-versus-\(8\) shift, the highest-contributing node scopes concentrate around the central image region where the additional stroke distinguishes an \(8\) from a \(0\), rather than along the outer contour shared by both digits. For the \(0\)-versus-\(6\) shift, the highlighted scopes concentrate around the central and upper-right regions that distinguish the two digit shapes. These examples show that the node-wise decomposition of HLD can provide interpretable localization of the variables associated with a detected distribution shift.

\begin{figure}
\centering
\includegraphics[width=0.45\columnwidth]{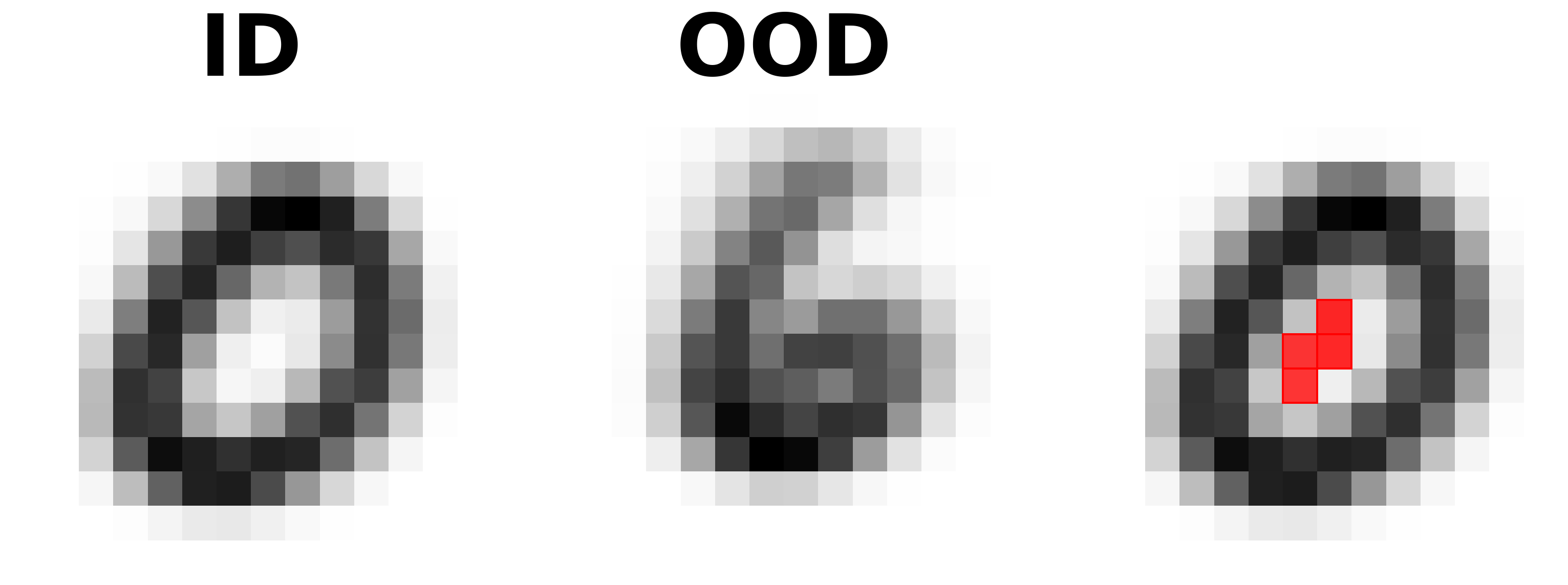}
\hspace{0.03\columnwidth}
\includegraphics[width=0.45\columnwidth]{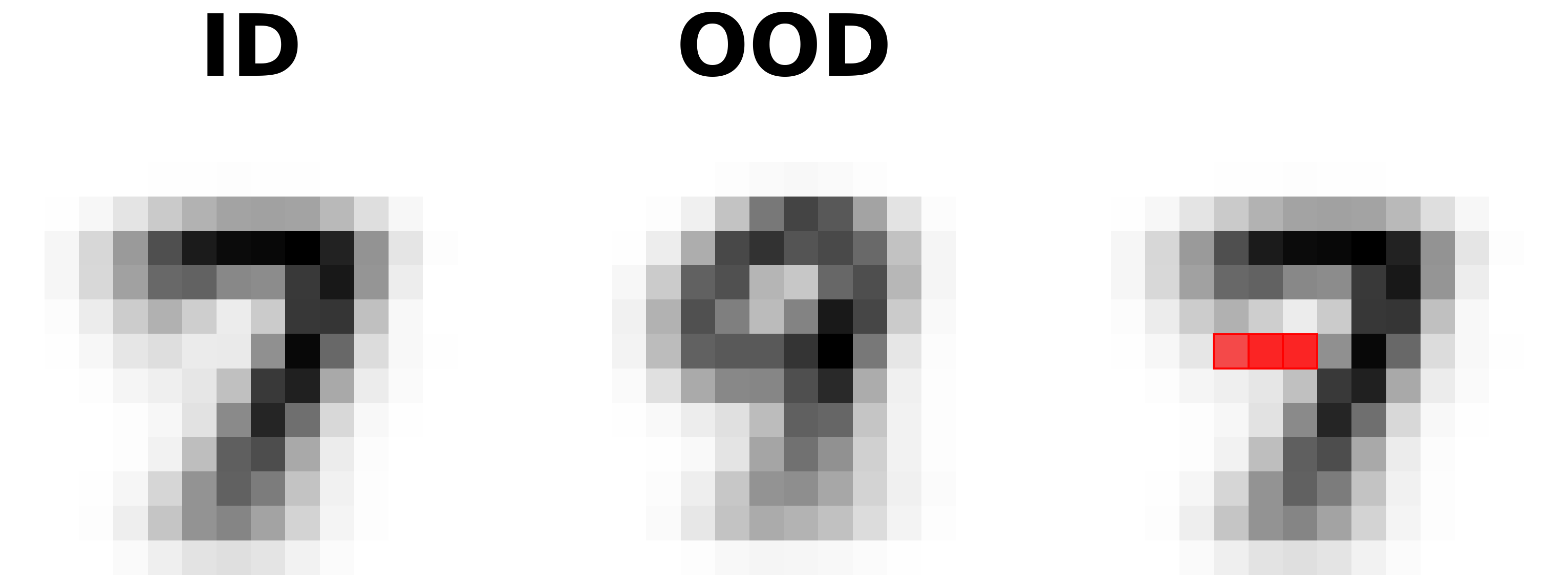}\\
\includegraphics[width=0.45\columnwidth]{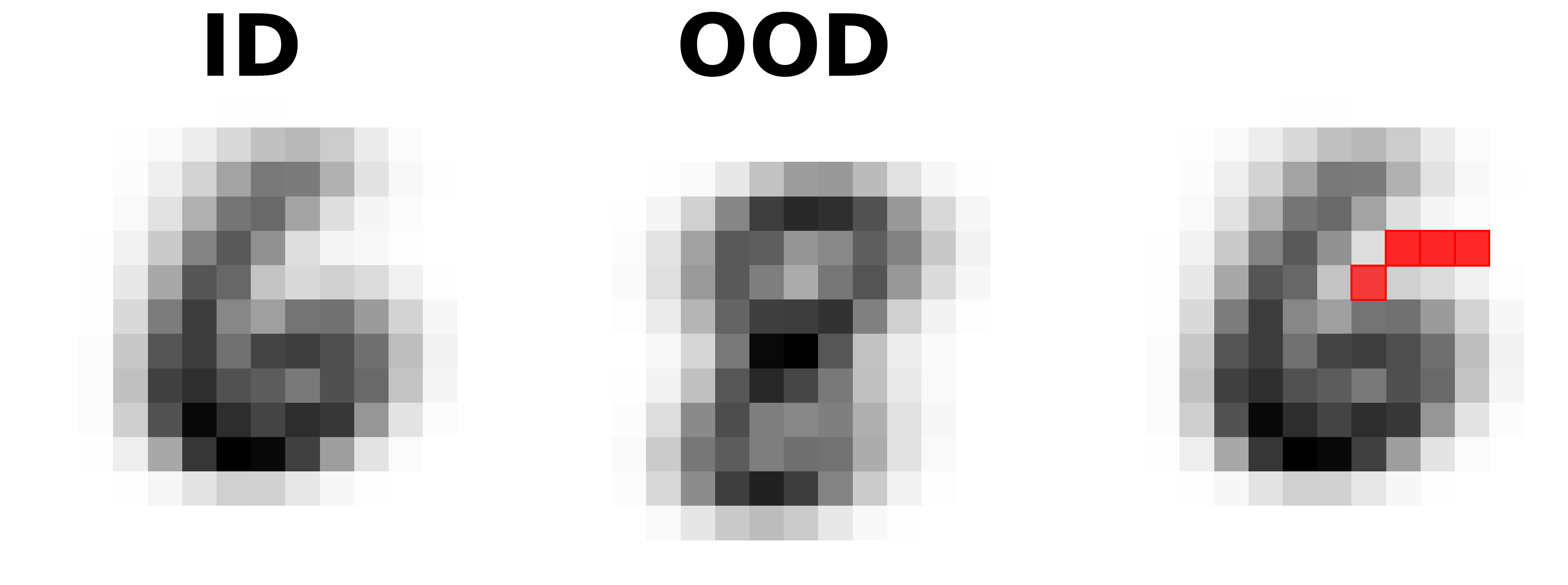}
\hspace{0.03\columnwidth}
\includegraphics[width=0.45\columnwidth]{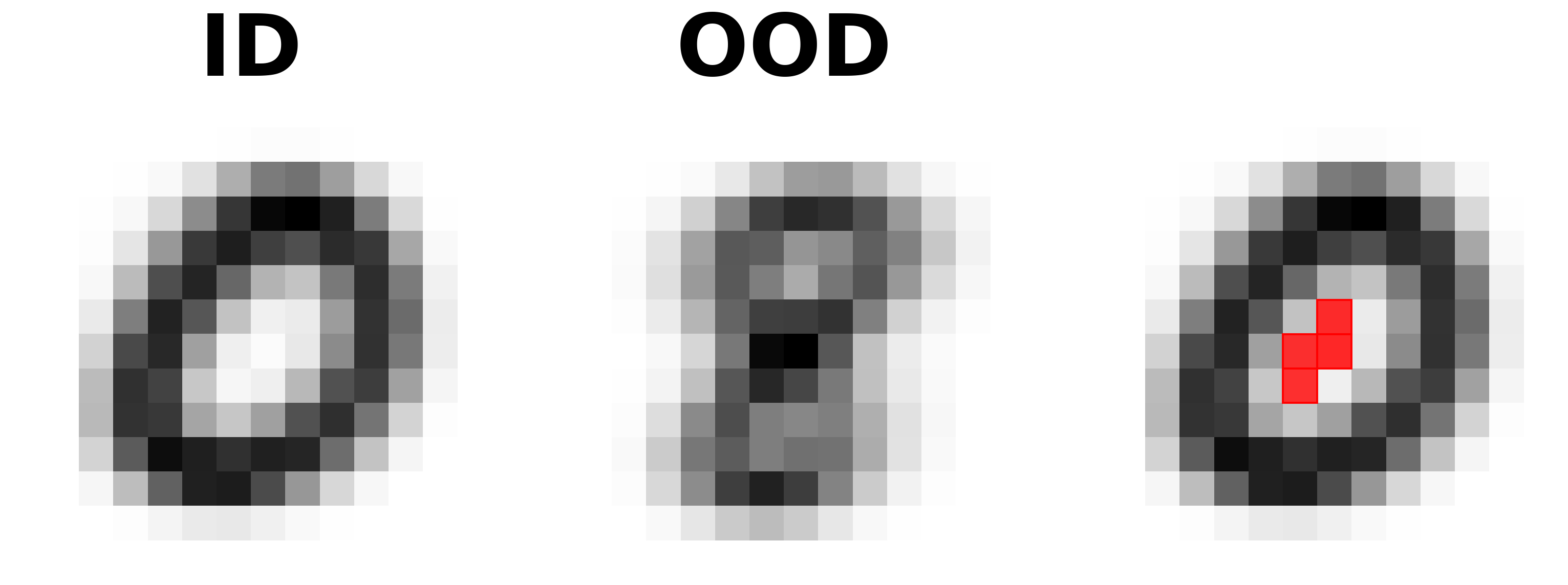}\\
\caption{ Each panel shows, for one ID--OOD digit pair, the top-ranked nodes scopes (highlighted) based on the HLD.}
\label{fig:comparison}
\end{figure}

\section{Conclusion, Limitations and Future Work}

We introduce the Hierarchical Likelihood Vector (HLV), which
represents an instance through likelihoods evaluated at selected
nodes of a PC and the Hierarchical Likelihood Distance (HLD), a PC-induced pseudo-metric between distributions based on their expected HLVs. We show that HLD is an Integral Probability Metric over a function class induced by the circuit and use it to construct a batch-level goodness-of-fit test for distributional OOD detection. For smooth, structured-decomposable PCs, the population mean and covariance of the HLV is computed exactly from the PC, enabling an approximate closed-form decision threshold without held-out ID reference or calibration data.

Across five tabular datasets and binarized MNIST (smaller resolution), HLD achieve high detection power at small and moderate batch sizes and exhibits lower variability across distribution shifts than the baselines. The false-positive rate on held-out ID data also increases more slowly at large batch sizes, indicating greater robustness to model misspecification in the evaluated settings. The node-type ablation shows that combining leaf and sum nodes provides the most reliable HLV representation, while the node-wise decomposition of the test statistic provides qualitative localization of the variables associated with distribution shifts.

The method nevertheless has two limitations. First, the statistical calibration is defined with respect to the distribution represented by the trained PC; false positives on real ID data therefore depend on the fidelity of the learned model (as illustrated by the results on full-resolution MNIST). Second, the closed-form threshold relies on a two-moment Gaussian approximation to a generalized chi-square distribution and may be less accurate when a small number of eigenvalues dominate the covariance spectrum. The generalized chi-square quantile can instead be computed numerically from the exactly obtained covariance matrix. Extending HLD to continuous input distributions, for which node densities need not be bounded, remains an important direction for future work.


\bibliography{aaai2027}

\appendix

\appendix
\onecolumn

\noindent \hrulefill
\section*{A Probabilistic Circuit-Induced Pseudo-Metric for Out-of-Distribution Detection}

\noindent \hrulefill
\vspace{1em}

\section{HLD as an IPM and Connection to Other Divergence Measures}
Integral Probability Metrics (IPMs) provide a general framework for measuring
discrepancies between probability distributions through expectations over a
prescribed class of test functions. Given a function class $\mathcal{F}$, the IPM
between two probability distributions $P$ and $Q$ is defined as \cite{mul97, sri12}
\[
\mathrm{IPM}_{\mathcal{F}}(P,Q)=\sup_{f\in\mathcal{F}}\left|\mathbb{E}_{P}[f(\mathbf{X})]-\mathbb{E}_{Q}[f(\mathbf{X})]\right|.
\]
Intuitively, an IPM measures the largest discrepancy in expectation that can be
achieved by any function in the class $\mathcal{F}$. Consequently, the choice of
$\mathcal{F}$ determines the discriminative power and properties of the resulting
statistical discrepancy. For example, the Maximum Mean Discrepancy (MMD)
\cite{gret12} is obtained when $\mathcal{F}$ is the unit ball of a reproducing
kernel Hilbert space, the Wasserstein-1 distance corresponds to the class of
1-Lipschitz functions \cite{vil09}, and the Total Variation distance is induced
by the class of bounded functions with $\|f\|_{\infty}\le 1$ \cite{sri12}.

Our proposed HLD compares two probability distributions through the expected
likelihood-vector representation induced by a probabilistic circuit. Like
classical measures such as the Kullback--Leibler (KL) divergence \cite{tsy09},
Total Variation (TV) distance \cite{tsy09}, Wasserstein distance \cite{vil09},
and Maximum Mean Discrepancy (MMD) \cite{gret12}, it quantifies the discrepancy
between two distributions, but through a different representation of the
underlying data.

\paragraph{HLD as an IPM.}
The HLD can itself be viewed through the IPM lens. Each coordinate of the
likelihood vector, $\ell_n(\mathbf{x})=p_n(\mathbf{x}_{S_n})$, is the probability
assigned to $\mathbf{x}_{S_n}$ by the subcircuit rooted at node $n$, and therefore
$\ell_n(\mathbf{x})\in[0,1]$. Writing $\boldsymbol{\mu}_P=\mathbb{E}_{P}[\mathbf{L}(\mathbf{X})]$
and $\boldsymbol{\mu}_Q=\mathbb{E}_{Q}[\mathbf{L}(\mathbf{X})]$, the HLD is the
$\ell_2$ distance between these two mean embeddings,
\[
d_{\mathcal{C}}(P,Q)=\|\boldsymbol{\mu}_P-\boldsymbol{\mu}_Q\|_2
=\Big(\textstyle\sum_{n=1}^{m}\big(\mathbb{E}_{P}[\ell_n]-\mathbb{E}_{Q}[\ell_n]\big)^2\Big)^{1/2},
\]
so each coordinate is an IPM with the single test function $\ell_n$, and the HLD
aggregates these coordinate-wise discrepancies over the $m$ selected nodes. In
this sense the HLD is a finite-dimensional, circuit-induced embedding
discrepancy: rather than optimizing over a function class, it fixes the test
functions to the node-level likelihood responses that the PC already computes.

\paragraph{Connection to total variation and KL.}
Because each test function is bounded, $\ell_n\in[0,1]$, every coordinate is
controlled by the total-variation distance. Using
$\mathrm{TV}(P,Q)=\sup_{0\le f\le 1}\left|\mathbb{E}_{P}[f]-\mathbb{E}_{Q}[f]\right|$,
we have $\big|\mathbb{E}_{P}[\ell_n]-\mathbb{E}_{Q}[\ell_n]\big|\le \mathrm{TV}(P,Q)$
for every node $n$. Summing the squared coordinates over the $m$ nodes gives
\[
d_{\mathcal{C}}(P,Q)^2=\sum_{n=1}^{m}\big(\mathbb{E}_{P}[\ell_n]-\mathbb{E}_{Q}[\ell_n]\big)^2
\;\le\; m\,\mathrm{TV}(P,Q)^2 .
\]
Combining this with Pinsker's inequality \cite{tsy09},
$\mathrm{TV}(P,Q)^2\le \tfrac{1}{2}\,\mathrm{KL}(P\|Q)$, yields
\[
d_{\mathcal{C}}(P,Q)^2\;\le\;\frac{m}{2}\,\mathrm{KL}(P\|Q),
\qquad\text{equivalently}\qquad
d_{\mathcal{C}}(P,Q)\;\le\;\sqrt{\tfrac{m}{2}\,\mathrm{KL}(P\|Q)} .
\]
These bounds establish a theoretical connection between the proposed divergence
and the classical statistical divergences. In particular, distributions that are
close in total variation or KL divergence are also close in the likelihood-vector
representation induced by the PC.

\section{Scope Trichotomy}
\begin{lemma}[Scope Trichotomy]
Let \(\mathcal C\) be a structured-decomposable and smooth PC with respect to a v-tree \(\mathcal T\). For any two nodes \(n_i,n_j \in \mathcal C\) with scopes \(S_i, S_j\) exactly one of the following holds: \(S_i\cap S_j =\emptyset \) or \(S_i=S_j\) or one of \(S_i, S_j\) strictly contains the other.    
\label{lemma:scope}
\end{lemma}
\begin{proof}
For a v-tree node $t$ let $\phi(t)$ be its scope, i.e.\ the variables at the leaves
below $t$. Two facts about a v-tree suffice:
(i) every $\phi(t)$ is nonempty; and
(ii) the two children of an internal node have nonempty, disjoint scopes whose
union is $\phi(t)$.

Because $\mathcal{C}$ is smooth and structured-decomposable with respect to $T$, each node's scope equals $\phi(t)$ for some vtree node $t$; let $t_i,t_j$ be the v-tree nodes with $S_i=\phi(t_i)$, $S_j=\phi(t_j)$. It is enough to relate $\phi(t_i)$ and $\phi(t_j)$ and in a rooted tree $t_i,t_j$ are either equal, in an ancestor--descendant relation, or incomparable.

\emph{Equal:} $t_i=t_j$ gives $S_i=S_j$.

\emph{Ancestor--descendant:} say $t_i$ is a proper ancestor of $t_j$. Then
$\phi(t_j)\subseteq\phi(t_i)$ and taking the child $c$ of $t_i$ that does
\emph{not} contain $t_j$, fact (ii) gives a variable in $\phi(c)\subseteq S_i$
that lies outside $S_j$; hence $S_j\subsetneq S_i$ (strict containment).

\emph{Incomparable:} let $u$ be their least common ancestor. Then $t_i,t_j$ sit
below different children of $u$, whose scopes are disjoint by (ii); since
$S_i,S_j$ are contained in these two scopes, $S_i\cap S_j=\varnothing$.

These three tree relations are mutually exclusive and exhaustive, giving exactly
one of $S_i=S_j$, strict containment, or $S_i\cap S_j=\varnothing$. (They cannot
coincide: by (i) the scopes are nonempty, so equality and strict containment both
exclude disjointness.)
\end{proof}

\section{Experiments}

\subsection{Datasets}
\subsubsection{UCI tabular}We evaluate on five UCI classification datasets: Adult, connect-4, Covertype, Census-KDD and Sensorless Drive Diagnosis \cite{bec96,bla98,tro95,cen00,bat13}. For each experiment, one class is treated as the in-
distribution, while each remaining class is considered as a
separate OOD distribution. This produces multiple ID–OOD
pairs for each dataset as listed in Table \ref{tab:uci-composition}. we discretize the continuous columns. Each ID class pool is split $50/50$ into training set A and a held-out set B, under the deployment protocol the trained circuit is the only in-distribution representation, so B is never used for training, threshold calibration or as a reference, it serves solely as a real-ID sanity check. 
\begin{table}[!ht]
\centering
\begin{tabular}{lllrrrrr}
\toprule
Dataset & ID class & OOD class & $d$ & $N_{\text{ID}}$ & $N_A$ & $N_B$ & $T$ \\
\midrule
Adult       & $\leq$50K & $>$50K & 14 & 37{,}155 & 18{,}577 & 18{,}578 & 2{,}500 \\
Adult       & $>$50K & $\leq$50K & 14 & 11{,}687 & 5{,}843 & 5{,}844 & 2{,}500 \\
\midrule
Connect-4   & win & loss & 42 & 44{,}473 & 22{,}236 & 22{,}237 & 2{,}500 \\
Connect-4   & win & draw & 42 & 44{,}473 & 22{,}236 & 22{,}237 & 2{,}500 \\
Connect-4   & loss & win & 42 & 16{,}635 & 8{,}317 & 8{,}318 & 2{,}500 \\
\midrule
Covertype   & lodgepole & aspen & 54 & 283{,}301 & 141{,}650 & 141{,}651 & 2{,}500 \\
Covertype   & lodgepole & spruce/fir & 54 & 283{,}301 & 141{,}650 & 141{,}651 & 2{,}500 \\
Covertype   & ponderosa & cottonwood & 54 & 35{,}754 & 17{,}877 & 17{,}877 & 2{,}500 \\
Covertype   & ponderosa & Douglas fir & 54 & 35{,}754 & 17{,}877 & 17{,}877 & 2{,}500 \\
\midrule
Census-KDD  & $<$50K & $\geq$50K & 40 & 280{,}717 & 140{,}358 & 140{,}359 & 2{,}500 \\
Census-KDD  & $\geq$50K & $<$50K & 40 & 18{,}568 & 9{,}284 & 9{,}284 & 2{,}500 \\
\midrule
Sensorless  & healthy & fault-02 & 48 & 5{,}319 & 2{,}659 & 2{,}660 & 2{,}500 \\
Sensorless  & healthy & fault-03 & 48 & 5{,}319 & 2{,}659 & 2{,}660 & 2{,}500 \\
Sensorless  & healthy & fault-04 & 48 & 5{,}319 & 2{,}659 & 2{,}660 & 2{,}500 \\
Sensorless  & healthy & fault-05 & 48 & 5{,}319 & 2{,}659 & 2{,}660 & 2{,}500 \\
Sensorless  & healthy & fault-06 & 48 & 5{,}319 & 2{,}659 & 2{,}660 & 2{,}500 \\
Sensorless  & healthy & fault-07 & 48 & 5{,}319 & 2{,}659 & 2{,}660 & 2{,}500 \\
Sensorless  & healthy & fault-08 & 48 & 5{,}319 & 2{,}659 & 2{,}660 & 2{,}500 \\
Sensorless  & healthy & fault-09 & 48 & 5{,}319 & 2{,}659 & 2{,}660 & 2{,}500 \\
Sensorless  & healthy & fault-10 & 48 & 5{,}319 & 2{,}659 & 2{,}660 & 2{,}500 \\
Sensorless  & healthy & fault-11 & 48 & 5{,}319 & 2{,}659 & 2{,}660 & 2{,}500 \\
\bottomrule
\end{tabular}
\caption{UCI dataset ID--OOD pair.}
\label{tab:uci-composition}
\end{table}
\subsubsection{Image} For MNIST we use all $90$ ordered digit pairs, treating each digit in turn as the in-distribution (ID) class and each of the other nine as an out-of-distribution (OOD). Images are trained in $28\times28$ resolution and also by downscaling to $7\times7$, a coarser resolution by average-pooling, flattened to a vector of pixels and discretized.\\

\subsection{Training} For every ID-OOD pair we train a Hidden Chow-Liu tree (HCLT) with four latent states per variable over the ID training split, optimized by EM for 100 epochs. We stop at 100 epochs because the training log-likelihood plateaus well before this point (Figure \ref{fig:epoch}), further epochs yield no significant improvement.

The HCLT does not learn the full-resolution image distribution well. Figure \ref{fig:MD} shows this across several digits, at the native $28{\times}28$ resolution, samples from the trained circuit lose its fine stroke topology, appearing fragmented and noisy relative to real examples. This weak fit at full resolution is the source of the elevated FPR observed at $28{\times}28$.
\begin{figure}
\centering
\includegraphics[width=0.5\columnwidth]{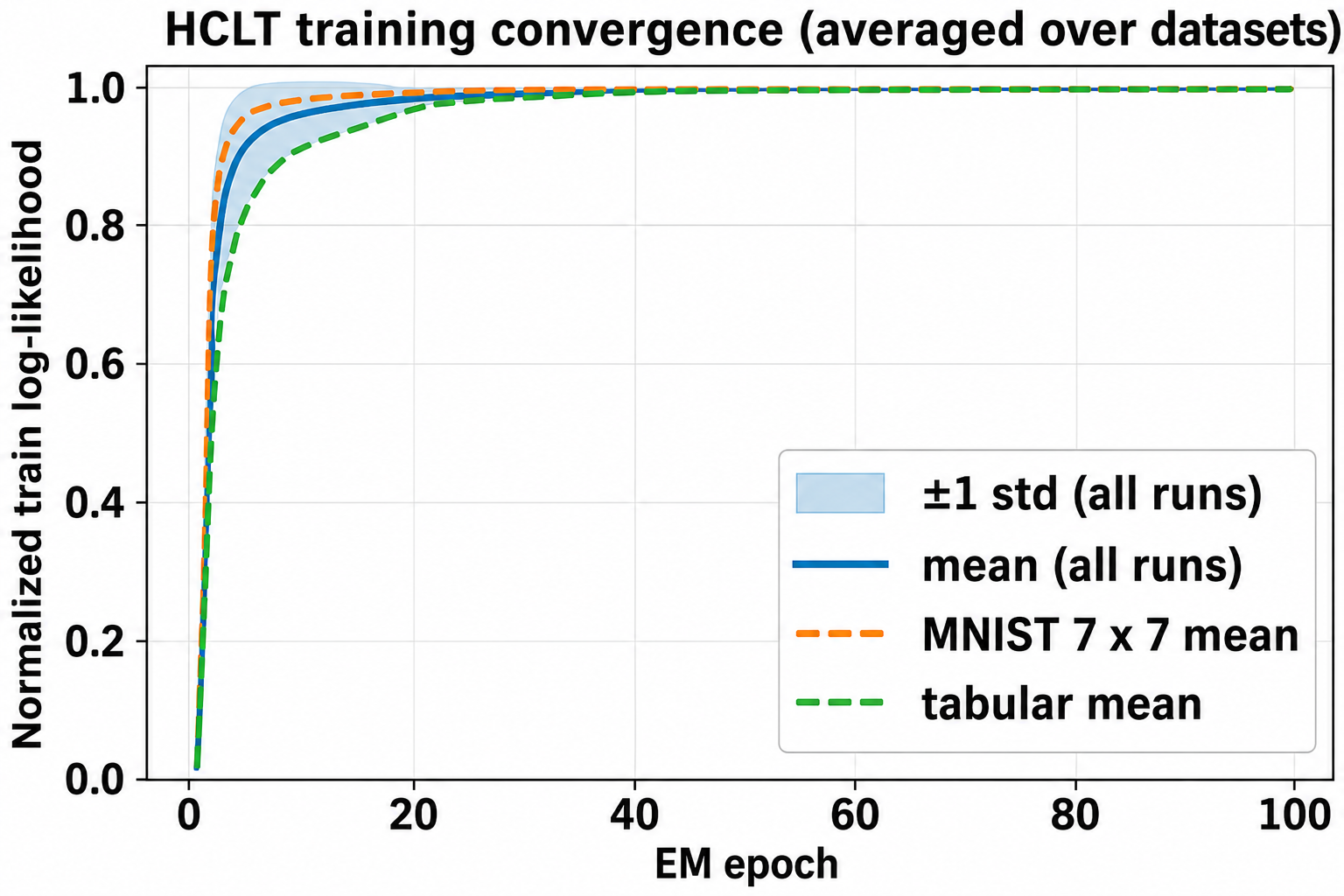}
\caption{HCLT training convergence: normalized training log-likelihood versus
EM epoch.}
\label{fig:epoch}
\end{figure}

\begin{figure}
\centering
\includegraphics[width=0.49\columnwidth]{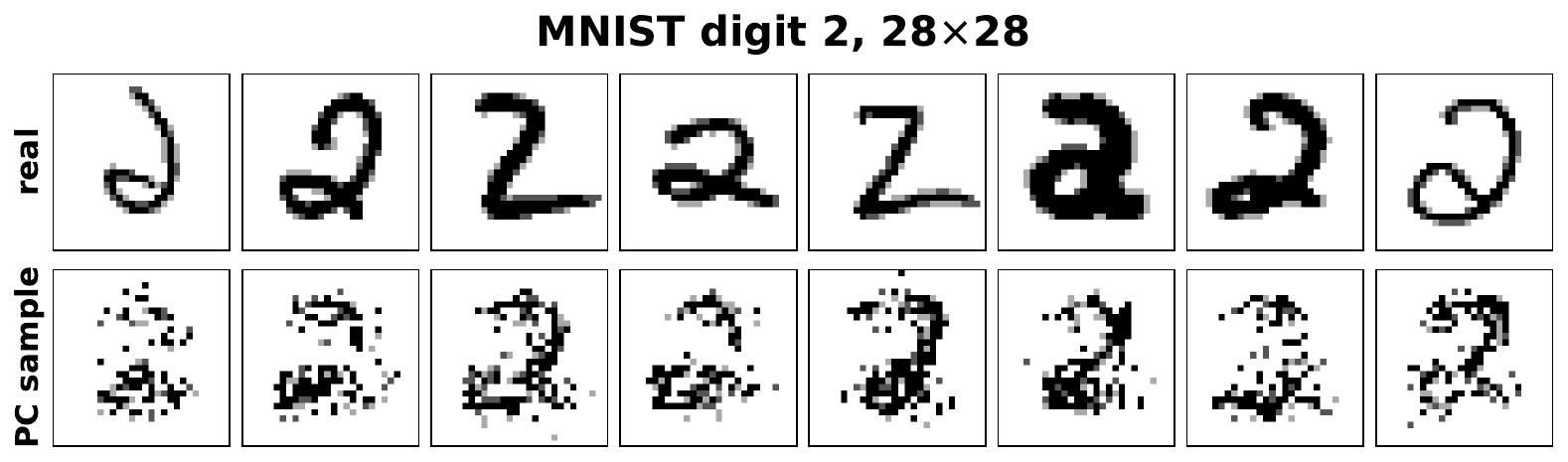}
\hfill
\includegraphics[width=0.49\columnwidth]{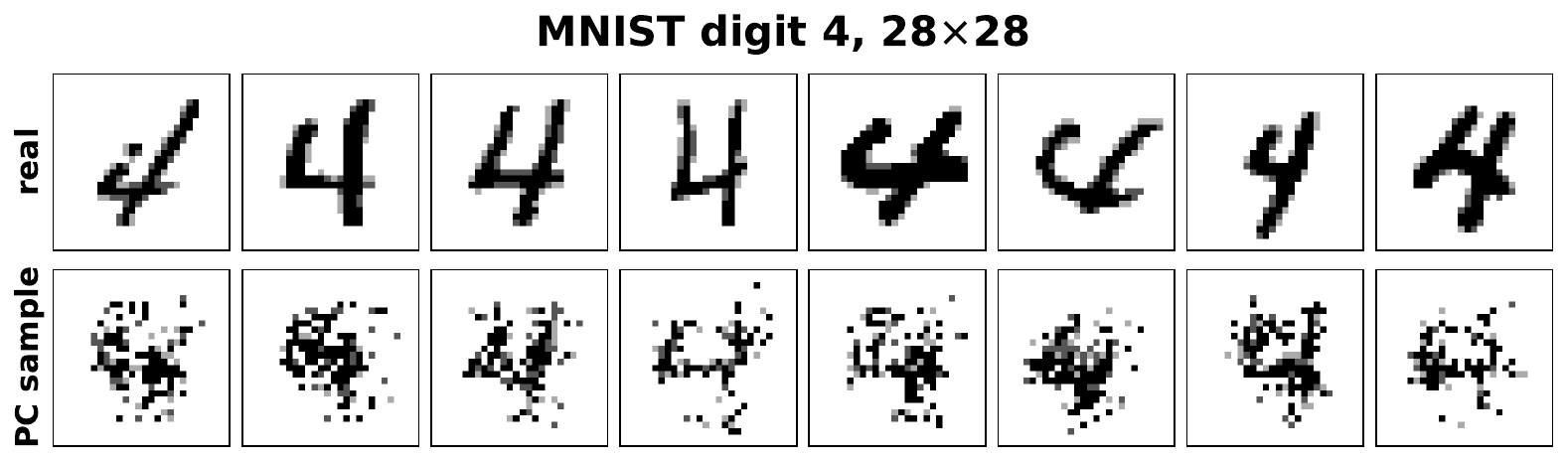}\\
\includegraphics[width=0.49\columnwidth]{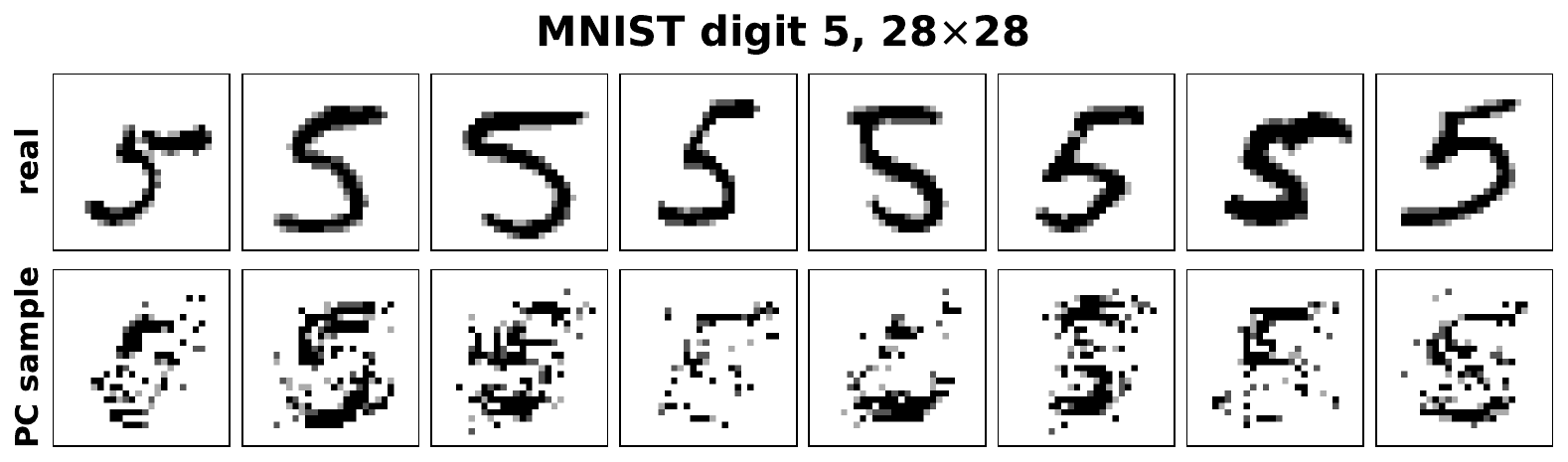}
\hfill
\includegraphics[width=0.49\columnwidth]{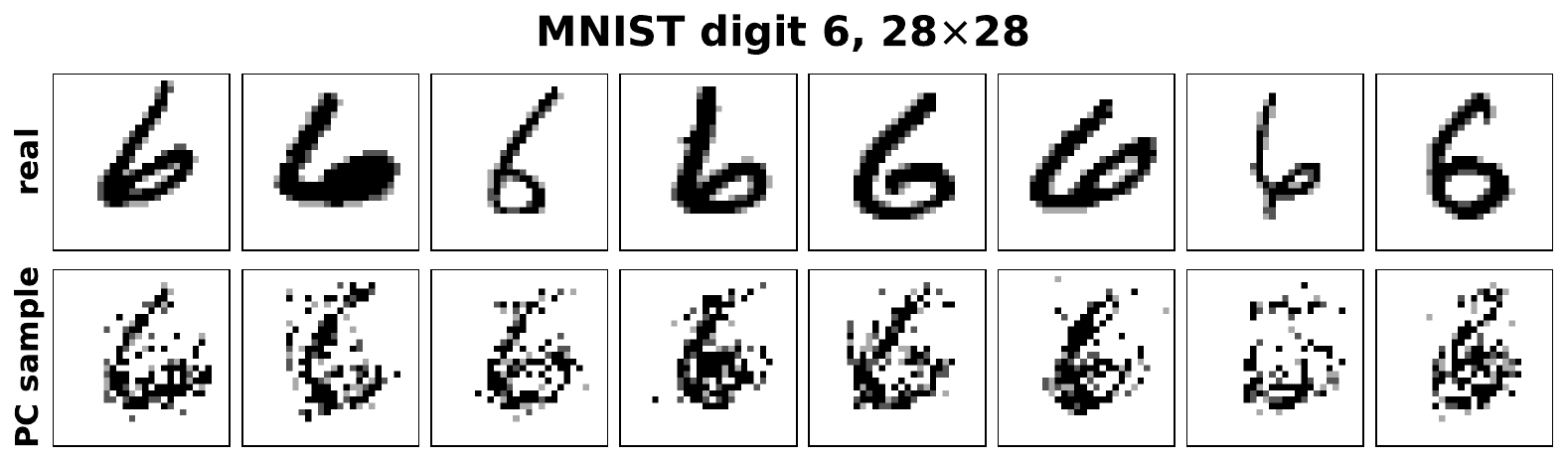}\\
\includegraphics[width=0.49\columnwidth]{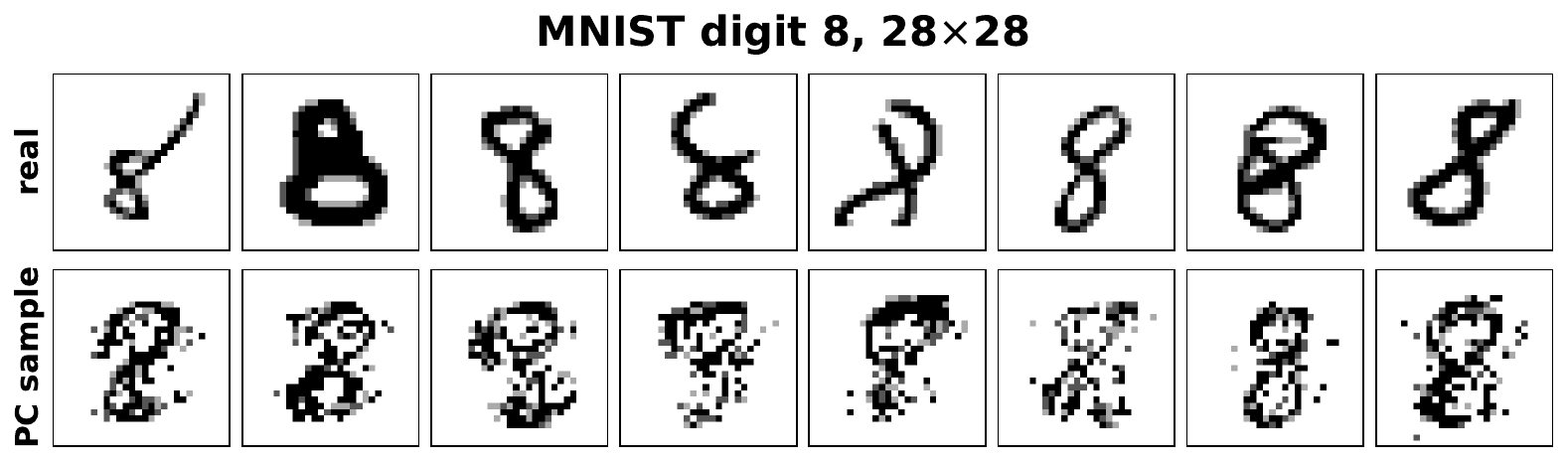}
\hfill
\includegraphics[width=0.49\columnwidth]{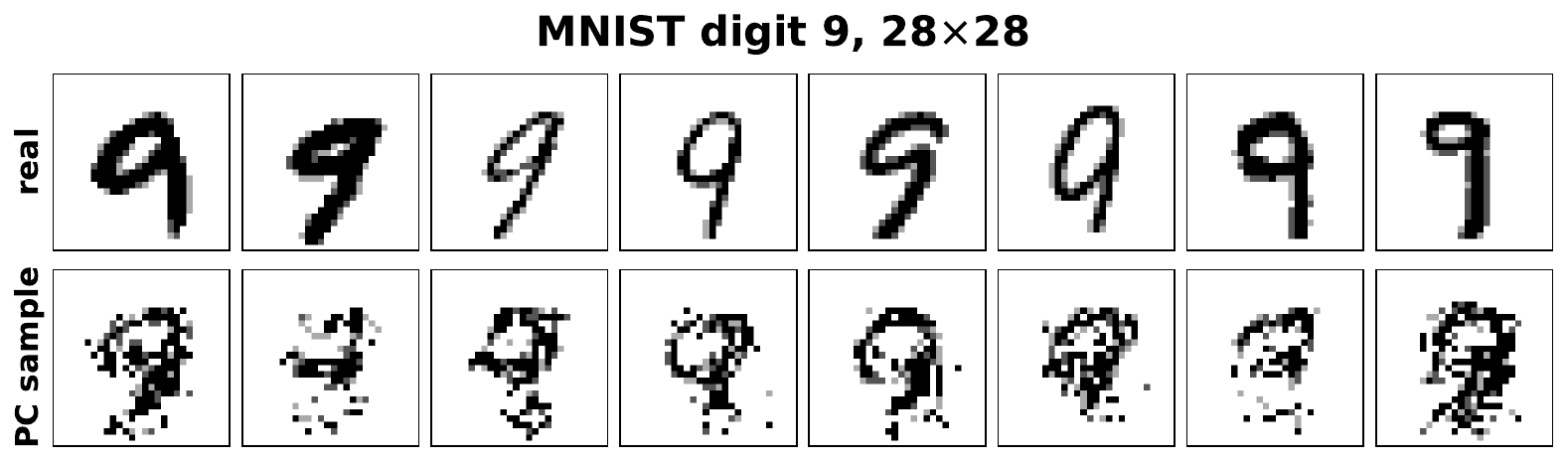}

\caption{Real MNIST digits (top) versus HCLT samples (bottom) at
$28{\times}28$ resolution. }
\label{fig:MD}
\end{figure}

\subsection{Baselines}
For a fair comparison, we cast every baseline as a batch level two sample test evaluated at the same batch size $T$ under the same calibration. We compare:
\begin{itemize}[wide, labelindent=0pt]
    \item \textbf{Typicality}: The absolute deviation of the test batch's mean negative log likelihood from the model's empirical entropy $\hat H$. For a test batch $B=\{x_1,\ldots,x_n\}$, the statistic is \cite{nal19d} \[T_{\mathrm{typ}}=\left|-\frac{1}{n}\sum_{i=1}^{n}\log p(x_i)-\hat{H}\right|.\]
    \item \textbf{Root LL}: The difference between the training heldout average and the test batch average negative log-likehood \cite{nal19},
    \[T_{\mathrm{root}}=\left|\frac{1}{n}\sum_{\mathbf x\in ID}\log p(\mathbf x)-\frac{1}{n}\sum_{\mathbf x\in test}\log p(\mathbf x)\right|.\]
    \item \textbf{TDI}: For every sample, we get a uncertainty score propagating dropout analytically through the PC. Let $u(\mathbf x)$ denote the resulting scalar TDI score. The batch statistic is\cite{ven23}:
    \[T_{TDI}=\left| \frac{1}{n}\sum_{\mathbf x \in ID}u(\mathbf x)-\frac{1}{n}\sum_{\mathbf x\in test} u(\mathbf x)\right|.\]
    \item \textbf{MMD}:Using a RBF kernel,
    \[k(\mathbf x, \mathbf y)=\exp\left(-\frac{\|\mathbf x-\mathbf y\|^2}{2\sigma^2}\right),\]
    where the $\sigma$ is selected using the median heuristic on the ID holdout set, the empirical statistic is \cite{gret12},
    \[MMD^2(B_1,B_2)=\frac{1}{n^2}\sum_{i,j}k(x_i,x_j)+\frac{1}{m^2}\sum_{i,j}k(y_i,y_j)-\frac{2}{nm}\sum_{i,j}k(x_i,y_j),\]
    where $B_1$ and $B_2$ denote the ID and test batches.
    \item \textbf{HLD}: The empirical HLD between the learned distribution $P$ and the test distribution $Q$ is then
\[
\hat{d}_{\mathrm{C}}(P,Q)=\|\boldsymbol{\mu}_P-\boldsymbol{\hat{\mu}}_Q\|_2,
\]
where $\boldsymbol{\mu}_P=\mathbb{E}_{\textbf{X}\sim P}[\textbf{L(X)}]$ denotes the population HLV under the distribution represented by the trained PC.
\end{itemize}
\subsection{HLD Algorithm}
Algorithms~\ref{alg:hld-offline} and~\ref{alg:hld-test} provide the pseudo-code for the goodness-of-fit test using HLD. Algorithm~\ref{alg:hld-offline} is \emph{offline}: the population mean $\boldsymbol{\mu}_P$, covariance $\boldsymbol{\Sigma}_P$, and the traces $\tr(\boldsymbol{\Sigma}_P)$, $\tr(\boldsymbol{\Sigma}_P^2)$ are computed once after the PC is trained. Algorithm~\ref{alg:hld-test} runs \emph{per test batch} at deployment: the
threshold $\tau$ depends on the batch size $T$ and is therefore evaluated at test
time from the precomputed traces, after which the HLD statistic is compared against it.
\begin{algorithm}
\caption{HLD offline pre-computation (after training)}
\label{alg:hld-offline}
\textbf{Input}: structured-decomposable PC $\mathcal{C}$ over $\mathbf{X}$ inducing
v-tree $\mathcal{T}$ (the sole model of the in-distribution $P$); selected nodes
$n_1,\dots,n_m$ with scopes $S_1,\dots,S_m$\\
\textbf{Output}: mean $\boldsymbol{\mu}_P$, covariance $\boldsymbol{\Sigma}_P$,
traces $\tr(\boldsymbol{\Sigma}_P)$, $\tr(\boldsymbol{\Sigma}_P^2)$
\begin{algorithmic}[1]
\STATE \textit{// Phase 1: exact population mean $\boldsymbol{\mu}_P$}
\FOR{$i = 1$ \TO $m$}
    \STATE marginalize $\mathcal{C}$ onto $S_i$ to obtain $p_{S_i}$
    \STATE $\mu_{P i} \gets \sum_{\mathbf{x}_{S_i}} \big(p_{S_i}\otimes p_{n_i}\big)(\mathbf{x}_{S_i})$
\ENDFOR
\STATE \textit{// Phase 2: exact covariance $\boldsymbol{\Sigma}_P$}
\FOR{$i = 1$ \TO $m$}
    \FOR{$j = i$ \TO $m$}
        \STATE $S_{ij} \gets$ scope of the LCA of the v-tree nodes for $S_i,S_j$
        \STATE marginalize $\mathcal{C}$ onto $S_{ij}$ to obtain $p_{S_{ij}}$
        \STATE $E_{ij} \gets \sum_{\mathbf{x}_{S_{ij}}}\big(p_{S_{ij}}\otimes p_{n_i}\otimes p_{n_j}\big)(\mathbf{x}_{S_{ij}})$
        \STATE $(\boldsymbol{\Sigma}_P)_{ij} \gets E_{ij}-\mu_{P i}\,\mu_{P j}$
        \STATE $(\boldsymbol{\Sigma}_P)_{ji} \gets (\boldsymbol{\Sigma}_P)_{ij}$
    \ENDFOR
\ENDFOR
\STATE $\tr(\boldsymbol{\Sigma}_P)\gets\sum_i(\boldsymbol{\Sigma}_P)_{ii}$;\quad
       $\tr(\boldsymbol{\Sigma}_P^2)\gets\sum_{i,j}(\boldsymbol{\Sigma}_P)_{ij}^2$
\STATE \textbf{return} $\boldsymbol{\mu}_P,\ \boldsymbol{\Sigma}_P,\
       \tr(\boldsymbol{\Sigma}_P),\ \tr(\boldsymbol{\Sigma}_P^2)$
\end{algorithmic}
\end{algorithm}
\begin{algorithm}
\caption{HLD batch-level OOD test (deployment)}
\label{alg:hld-test}
\textbf{Input}: selected nodes $n_1,\dots,n_m$ with scopes $S_1,\dots,S_m$;
precomputed $\boldsymbol{\mu}_P$, $\tr(\boldsymbol{\Sigma}_P)$,
$\tr(\boldsymbol{\Sigma}_P^2)$ (Algorithm~\ref{alg:hld-offline}); test batch
$\mathcal{B}=\{\mathbf{x}^{(1)},\dots,\mathbf{x}^{(T)}\}$\\
\textbf{Parameter}: batch size $T$; significance level $\alpha$\\
\textbf{Output}: decision $\textsc{OOD}$ or $\textsc{ID}$
\begin{algorithmic}[1]
\STATE \textit{// Phase 1: decision threshold from the null distribution}
\STATE $\tau \gets 
\sqrt{
\frac{\operatorname{tr}(\boldsymbol{\Sigma}_P)}{T}
+
\frac{z_{1-\alpha}}{T}
\sqrt{2\operatorname{tr}(\boldsymbol{\Sigma}_P^2)}
}$
\STATE \textit{// Phase 2: test the batch $\mathcal{B}$}
\FOR{$t = 1$ \TO $T$}
    \STATE $\mathbf{L}_\mathcal{C}(\mathbf{x}^{(t)}) \gets \big(p_{n_1}(\mathbf{x}^{(t)}_{S_1}),\dots,p_{n_m}(\mathbf{x}^{(t)}_{S_m})\big)$
\ENDFOR
\STATE $\boldsymbol{\hat{\mu}}_Q \gets \tfrac{1}{T}\sum_{t=1}^{T}\mathbf{L}_\mathcal{C}(\mathbf{x}^{(t)})$
\STATE $\Delta_T \gets \big\lVert \boldsymbol{\mu}_P - \boldsymbol{\hat{\mu}}_Q \big\rVert_2$ \COMMENT{HLD statistic $\widehat d_{\mathcal C}(P,Q)$}
\STATE \textbf{if} $\Delta_T > \tau$ \textbf{then return} $\textsc{OOD}$ \textbf{else return} $\textsc{ID}$
\end{algorithmic}
\end{algorithm}
\section{Ablations}

\subsection{Architecture}
We ablate on the capacity and architecture type of the PC. All runs use the same pipeline, the same batch test protocol on the adult dataset.

\paragraph{Capacity sweep (HCLT)}
We vary the number of latent states, $K\in{4,8,16}$, larger K yields a wider circuit and a longer likelihood vector. Table \ref{tab:arch-capacity} reports circuit size, training log-likelihood and detection rate. Increasing K improves the fit only marginally, log-likelihood rises from $-13.50(K=4)$ to $-13.08(K=16)$, a change too small to affect the HLV comparison and detection reflects this directly from $T=100$ onward, all three settings are saturated at essentially perfect detection, so added capacity yields no further discriminative benefit. The reading is that once additional latent size brings no new information as here, increasing $K$ will not help. Here $K=4$ suffices as detection is saturated.
\begin{table}[!ht]\centering
\begin{tabular}{lrrrrrr}
\toprule
$K$ & scalars & LL & $T{=}10$ & $T{=}50$ & $T{=}100$ & $T{=}200$  \\
\midrule
$4$ & $101$ & $-13.50$ & $0.104$ & $0.558$ & $1.000$ & $1.000$  \\
$8$ & $201$ & $-13.17$ & $0.080$ & $0.400$ & $0.998$ & $1.000$  \\
$16$ & $401$ & $-13.08$ & $0.054$ & $0.578$ & $0.998$ & $1.000$ \\
\bottomrule
\end{tabular}
\caption{HCLT capacity sweep. Circuit size and HLD detection rate, $\alpha{=}0.05$. }
\label{tab:arch-capacity}
\end{table}
\paragraph{Architecture Change - RAT-SPN}
We replace the HCLT with RAT-SPN, keeping the training and batch protocol identical. We sweep both dimensions of RAT-SPN capacity: the split depth $D\!\in\!\{2,3\}$ and the block size $K\!\in\!\{4,8\}$. Three things are visible from the table \ref{tab:arch-ratspn}. First, HLD attains near perfect detection rate under RAT-SPN as well. Second, Increasing depth from $D\!=\!2$ to $D\!=\!3$ improves detection rate uniformly and increasing the $K$ degrades the detection rate. Third, at matched parameter budget the, HCLT is more efficient, matches or exceeds every RAT-SPN configuration we tried. The Chow Liu tree used by HCLT gives HLV a structural inductive bias that RAT-SPN's random splits do not recover.
\begin{table}[!ht]\centering
\begin{tabular}{lrrrrrrr}
\toprule
 ($D,K$) & scalars & LL & $T{=}10$ & $T{=}50$ & $T{=}100$ & $T{=}200$  \\
\midrule
$(2,4)$ &  $109$ & $-14.88$ & $0.080$ & $0.294$ & $1.000$ & $1.000 $  \\
$(2,8)$ &  $217$ & $-14.80$ & $0.000$ & $0.222$ & $0.800$ & $0.984 $  \\
$(3,4)$ &  $173$ & $-14.70$ & $0.024$ & $0.684$ & $1.000$ & $1.000$  \\
$(3,8)$ & $345$ & $-14.49$ & $0.008$ & $0.392$ & $0.994$ & $1.000$  \\
\bottomrule
\end{tabular}
\caption{RAT-SPN sweep. Split depth $D$ and block size $K$; HLD detection rate, $\alpha{=}0.05$.}
\label{tab:arch-ratspn}
\end{table}
\subsection{Node Ablation}
The Tables \ref{tab:node-ablation-adult}, \ref{tab:node-ablation-connect4}, \ref{tab:node-ablation-censuskdd}, \ref{tab:node-ablation-covertype}, and \ref{tab:node-ablation-sensorless} show the detection rate on each tabular dataset based on the choice of nodes used in representing HLV. While the optimal node selection varies across the different datasets, we observe that choosing the sum and leaf nodes across all the datasets may be the best option.
\begin{table}[!ht]\centering
\begin{tabular}{lrrrrr}
\toprule
node set & $\dim$ & $T{=}1$ & $T{=}10$ & $T{=}50$ & $T{=}100$ \\
\midrule
leaf               &   $56$ & $\mathbf{0.157\pm0.043}$ & $0.891\pm0.085$ & $1.000\pm0.000$ & $1.000\pm0.000$\\
sum                &   $ 25$  & $ 0.061\pm0.031$  & $ 0.719\pm0.193$  & $ 1.000\pm0.000$  & $ 1.000\pm0.000$  \\
prod               &   $ 28$  & $ 0.094\pm0.004$  & $ 0.263\pm0.019$  & $ 1.000\pm0.000 $ & $ 1.000\pm0.000$  \\
leaf$+$sum         &   $ 81$  & $ 0.141\pm0.019$  & $ \mathbf{0.968\pm0.010}$  & $ 1.000\pm0.000$  & $ 1.000\pm0.000 $ \\
leaf$+$prod        &   $ 84$  & $ 0.076\pm0.034$ & $0.677\pm0.159$ & $1.000\pm0.000$ & $1.000\pm0.000$ \\
sum$+$prod         &   $53$ & $0.087\pm0.045$ & $0.396\pm0.018$ & $1.000\pm0.000$ & $1.000\pm0.000$ \\
leaf$+$sum$+$prod  &  $109$ & $0.126\pm0.010$ & $0.793\pm0.015$ & $1.000\pm0.000 $& $1.000\pm0.000 $\\
\bottomrule
\end{tabular}
\caption{Node-type ablation on \textbf{adult}.}
\label{tab:node-ablation-adult}
\end{table}

\begin{table}[!ht]\centering
\begin{tabular}{lrrrrr}
\toprule
node set & $\dim$ & $T{=}1$ & $T{=}10$ & $T{=}50$ & $T{=}100$ \\
\midrule
leaf               &  $168$ & $0.006\pm0.004$ & $\mathbf{0.113\pm0.010}$ & $0.695\pm0.261$ & $0.896\pm0.147$ \\
sum                &  $137$ & $\mathbf{0.021\pm0.009}$ & $0.047\pm0.059$ & $0.651\pm0.362$ & $\mathbf{0.973\pm0.038}$ \\
prod               &  $140$ & $0.008\pm0.004$ & $0.001\pm0.001$ & $0.151\pm0.213$ & $0.379\pm0.440$ \\
leaf$+$sum         &  $305$ & $0.007\pm0.005$ & $0.075\pm0.035$ & $\mathbf{0.763\pm0.290}$ & $0.955\pm0.064$ \\
leaf$+$prod        &  $308$ & $0.009\pm0.008$ & $0.012\pm0.011$ & $0.285\pm0.302$ & $0.703\pm0.263$ \\
sum$+$prod         &  $277$ & $0.018\pm0.005$ & $0.001\pm0.001$ & $0.067\pm0.092$ & $0.534\pm0.351$ \\
leaf$+$sum$+$prod  &  $445$ & $0.013\pm0.007$ & $0.003\pm0.003$ & $0.293\pm0.390$ & $0.850\pm0.186$ \\
\bottomrule
\end{tabular}
\caption{Node-type ablation on \textbf{connect4}.}
\label{tab:node-ablation-connect4}
\end{table}

\begin{table}[!ht]\centering
\begin{tabular}{lrrrrr}
\toprule
node set & $\dim$ & $T{=}1$ & $T{=}10$ & $T{=}50$ & $T{=}100$ \\
\midrule
leaf               &  $136$ & $\mathbf{0.056\pm0.028}$ & $\mathbf{0.453\pm0.007}$ & $1.000\pm0.000$ & $1.000\pm0.000$\\
sum                &   $57$ & $0.023\pm0.005$ & $0.068\pm0.048$ & $1.000\pm0.000 $& $1.000\pm0.000 $\\
prod               &   $60$ & $0.032\pm0.022$ & $0.024\pm0.012$ & $0.168\pm0.100$ & $0.721\pm0.302$ \\
leaf$+$sum         &  $193$ & $0.044\pm0.030$ & $0.441\pm0.239$ & $1.000\pm0.000$& $1.000\pm0.000$ \\
leaf$+$prod        &  $196$ & $0.027\pm0.043$ & $0.028\pm0.008$ & $0.582\pm0.018$ & $0.782\pm0.000$ \\
sum$+$prod         &  $117$ & $0.039\pm0.023$ & $0.038\pm0.010$ & $0.558\pm0.276$ & $1.000\pm0.001$ \\
leaf$+$sum$+$prod  &  $253$ & $0.049\pm0.023$ & $0.057\pm0.025$ & $0.998\pm0.002$ & $1.000\pm0.000$ \\
\bottomrule
\end{tabular}
\caption{Node-type ablation on \textbf{census-KDD}.}
\label{tab:node-ablation-censuskdd}
\end{table}

\begin{table}[!ht]\centering
\begin{tabular}{lrrrrr}
\toprule
node set & $\dim$ & $T{=}1$ & $T{=}10$ & $T{=}50$ & $T{=}100$ \\
\midrule
leaf               &   $40$ & $\mathbf{0.035\pm0.028}$ & $\mathbf{0.459\pm0.267}$ & $\mathbf{0.979\pm0.037}$ & $\mathbf{1.000\pm0.000}$ \\
sum                &   $25$ & $0.035\pm0.027$ & $0.041\pm0.042$ & $0.822\pm0.193$ & $0.973\pm0.048$ \\
prod               &   $28$ & $0.021\pm0.019$ & $0.014\pm0.011$ & $0.544\pm0.365$ & $0.902\pm0.170$ \\
leaf$+$sum         &   $65$ & $0.035\pm0.031$ & $0.237\pm0.114$ & $0.949\pm0.088$ & $0.999\pm0.002$ \\
leaf$+$prod        &   $68$ & $0.021\pm0.015$ & $0.025\pm0.011$ & $0.804\pm0.287$ & $0.987\pm0.025$ \\
sum$+$prod         &   $53$ & $0.028\pm0.022$ & $0.020\pm0.016$ & $0.652\pm0.327$ & $0.961\pm0.069$ \\
leaf$+$sum$+$prod  &   $93$ & $0.028\pm0.023$ & $0.029\pm0.021$ & $0.819\pm0.281$ & $0.988\pm0.021$ \\
\bottomrule
\end{tabular}
\caption{Node-type ablation on \textbf{covertype}.}
\label{tab:node-ablation-covertype}
\end{table}

\begin{table}[!ht]\centering
\begin{tabular}{lrrrrr}
\toprule
node set & $\dim$ & $T{=}1$ & $T{=}10$ & $T{=}50$ & $T{=}100$ \\
\midrule
leaf               &  $192$ & $\mathbf{0.092\pm0.039}$ & $\mathbf{0.968\pm0.056}$ & $1.000\pm0.000$ & $1.000\pm0.000$ \\
sum                &  $105$ & $0.022\pm0.018$ & $0.005\pm0.005$ & $0.000\pm0.000$ & $0.000\pm0.001$ \\
prod               &  $108$ & $0.019\pm0.012$ & $0.002\pm0.001$ & $0.001\pm0.001$ & $0.000\pm0.000$ \\
leaf$+$sum         &  $297$ & $0.021\pm0.014$ & $0.122\pm0.159$ & $0.508\pm0.492$ & $0.637\pm0.457$ \\
leaf$+$prod        &  $300$ & $0.019\pm0.012$ & $0.003\pm0.003$ & $0.004\pm0.006$ & $0.333\pm0.471$ \\
sum$+$prod         &  $213$ & $0.017\pm0.009$ & $0.006\pm0.006$ & $0.000\pm0.000$ & $0.000\pm0.000$ \\
leaf$+$sum$+$prod  &  $405$ & $0.019\pm0.012$ & $0.004\pm0.002$ & $0.000\pm0.000$ & $0.167\pm0.373$ \\
\bottomrule
\end{tabular}
\caption{Node-type ablation on \textbf{sensorless}.}
\label{tab:node-ablation-sensorless}
\end{table}

\subsection{$T$ and $\alpha$}

\paragraph{Effect of $\alpha$}
Tables \ref{tab:fpr1_a001}, \ref{tab:fpr2_a001}, \ref{tab:fpr1_a01}, and \ref{tab:fpr2_a01} show tightening the level (from $\alpha{=}0.1$ to $\alpha{=}0.01$ reduce the FPR uniformly, while loosening improves the FPR. The \(\mathrm{FPR}_{\mathrm{model}}\) tracks $\alpha$ closely, sitting near $0.10$ at $\alpha{=}0.1$ and near $0.01$ at $\alpha{=}0.01$, confirming all methods stay calibrated at either levels. \(\mathrm{FPR}_{\mathrm{data}}\) also shifts down under a tighter $\alpha$, at large $T$ values fall roughly by half (e.g HLD at $T{=}2500$ drops from $\approx 0.62$ at $\alpha=0.1$ to $\approx0.39$ at $\alpha{=}0.01$.

\paragraph{Effect of $T$}
Tables \ref{tab:fpr1_a001}, \ref{tab:fpr2_a001}, \ref{tab:fpr1_a01}, and \ref{tab:fpr2_a01} show \(\mathrm{FPR}_{\mathrm{model}}\) stays flat in $T$, pinned at the nominal level for every method. \(\mathrm{FPR}_{\mathrm{data}}\) instead climbs steadily as $T$ grows. Tightening $\alpha$ delays this climb but does not remove it, as the underlying model-data gap is unchanged.

HLD remains the best controlled of the five methods across both axes.
\begin{table}[!ht]
\centering
\begin{tabular}{rccccc}
\toprule
$T$ & HLD & MMD & RootLL & TDI & Typicality \\
\midrule
1 & $0.027 \pm 0.026$ & $0.004 \pm 0.003$ & $0.010 \pm 0.003$ & $0.011 \pm 0.004$ & $0.008 \pm 0.005$ \\
10 & $0.029 \pm 0.004$ & $0.010 \pm 0.003$ & $0.011 \pm 0.002$ & $0.011 \pm 0.001$ & $0.010 \pm 0.002$ \\
50 & $0.031 \pm 0.001$ & $0.010 \pm 0.001$ & $0.010 \pm 0.002$ & $0.009 \pm 0.002$ & $0.010 \pm 0.001$ \\
100 & $0.030 \pm 0.002$ & $0.009 \pm 0.001$ & $0.011 \pm 0.001$ & $0.010 \pm 0.001$ & $0.014 \pm 0.002$ \\
200 & $0.029 \pm 0.004$ & $0.011 \pm 0.002$ & $0.012 \pm 0.002$ & $0.011 \pm 0.002$ & $0.011 \pm 0.002$ \\
500 & $0.032 \pm 0.003$ & $0.014 \pm 0.005$ & $0.008 \pm 0.002$ & $0.009 \pm 0.002$ & $0.011 \pm 0.005$ \\
1000 & $0.029 \pm 0.005$ & $0.010 \pm 0.004$ & $0.012 \pm 0.002$ & $0.012 \pm 0.002$ & $0.010 \pm 0.004$ \\
1500 & $0.029 \pm 0.003$ & $0.009 \pm 0.003$ & $0.009 \pm 0.002$ & $0.008 \pm 0.002$ & $0.012 \pm 0.002$ \\
2000 & $0.030 \pm 0.004$ & $0.008 \pm 0.002$ & $0.014 \pm 0.003$ & $0.014 \pm 0.003$ & $0.011 \pm 0.005$ \\
2500 & $0.028 \pm 0.005$ & $0.011 \pm 0.001$ & $0.011 \pm 0.003$ & $0.012 \pm 0.004$ & $0.011 \pm 0.002$ \\
\bottomrule
\end{tabular}
\caption{\(\mathrm{FPR}_{\mathrm{data}}\) , averaged over 5 tabular datasets, $\alpha=0.01$.}
\label{tab:fpr1_a001}
\end{table}
\begin{table}[!ht]
\centering
\begin{tabular}{rccccc}
\toprule
$N$ & HLD & MMD & RootLL & TDI & Typicality \\
\midrule
1 & $0.015 \pm 0.020$ & $0.002 \pm 0.001$ & $0.008 \pm 0.005$ & $0.008 \pm 0.005$ & $0.007 \pm 0.006$ \\
10 & $0.027 \pm 0.004$ & $0.032 \pm 0.002$ & $0.024 \pm 0.008$ & $0.029 \pm 0.008$ & $0.021 \pm 0.007$ \\
50 & $0.027 \pm 0.002$ & $0.033 \pm 0.002$ & $0.040 \pm 0.005$ & $0.040 \pm 0.004$ & $0.031 \pm 0.010$ \\
100 & $0.031 \pm 0.006$ & $0.048 \pm 0.003$ & $0.045 \pm 0.009$ & $0.056 \pm 0.010$ & $0.042 \pm 0.025$ \\
200 & $0.039 \pm 0.015$ & $0.058 \pm 0.009$ & $0.056 \pm 0.016$ & $0.066 \pm 0.016$ & $0.052 \pm 0.026$ \\
500 & $0.061 \pm 0.032$ & $0.101 \pm 0.044$ & $0.098 \pm 0.043$ & $0.079 \pm 0.045$ & $0.123 \pm 0.104$ \\
1000 & $0.090 \pm 0.057$ & $0.217 \pm 0.125$ & $0.285 \pm 0.198$ & $0.278 \pm 0.197$ & $0.378 \pm 0.403$ \\
1500 & $0.192 \pm 0.108$ & $0.233 \pm 0.114$ & $0.294 \pm 0.312$ & $0.290 \pm 0.317$ & $0.445 \pm 0.420$ \\
2000 & $0.284 \pm 0.159$ & $0.352 \pm 0.099$ & $0.374 \pm 0.377$ & $0.366 \pm 0.384$ & $0.493 \pm 0.446$ \\
2500 & $0.394 \pm 0.210$ & $0.423 \pm 0.107$ & $0.423 \pm 0.404$ & $0.413 \pm 0.412$ & $0.523 \pm 0.462$ \\
\bottomrule
\end{tabular}
\caption{\(\mathrm{FPR}_{\mathrm{model}}\), averaged over 5 tabular datasets, $\alpha=0.01$.}
\label{tab:fpr2_a001}
\end{table}
\begin{table}[!ht]
\centering
\begin{tabular}{rccccc}
\toprule
$T$ & HLD & MMD & RootLL & TDI & Typicality \\
\midrule
1 & $0.111 \pm 0.062$ & $0.055 \pm 0.034$ & $0.093 \pm 0.009$ & $0.094 \pm 0.010$ & $0.100 \pm 0.006$ \\
10 & $0.096 \pm 0.001$ & $0.099 \pm 0.010$ & $0.097 \pm 0.010$ & $0.096 \pm 0.011$ & $0.100 \pm 0.010$ \\
50 & $0.102 \pm 0.010$ & $0.099 \pm 0.009$ & $0.095 \pm 0.011$ & $0.095 \pm 0.013$ & $0.098 \pm 0.017$ \\
100 & $0.104 \pm 0.003$ & $0.099 \pm 0.013$ & $0.101 \pm 0.009$ & $0.100 \pm 0.007$ & $0.103 \pm 0.009$ \\
200 & $0.102 \pm 0.014$ & $0.099 \pm 0.005$ & $0.098 \pm 0.006$ & $0.097 \pm 0.006$ & $0.093 \pm 0.006$ \\
500 & $0.106 \pm 0.008$ & $0.110 \pm 0.011$ & $0.102 \pm 0.008$ & $0.102 \pm 0.009$ & $0.101 \pm 0.004$ \\
1000 & $0.097 \pm 0.012$ & $0.099 \pm 0.008$ & $0.097 \pm 0.013$ & $0.095 \pm 0.012$ & $0.094 \pm 0.010$ \\
1500 & $0.102 \pm 0.009$ & $0.089 \pm 0.011$ & $0.105 \pm 0.008$ & $0.105 \pm 0.007$ & $0.101 \pm 0.008$ \\
2000 & $0.109 \pm 0.014$ & $0.098 \pm 0.012$ & $0.106 \pm 0.005$ & $0.106 \pm 0.004$ & $0.101 \pm 0.008$ \\
2500 & $0.102 \pm 0.009$ & $0.106 \pm 0.012$ & $0.104 \pm 0.015$ & $0.103 \pm 0.015$ & $0.101 \pm 0.011$ \\
\bottomrule
\end{tabular}
\caption{\(\mathrm{FPR}_{\mathrm{data}}\), averaged over 5 tabular datasets, $\alpha=0.1$.}
\label{tab:fpr1_a01}
\end{table}
\begin{table}[!ht]
\centering
\begin{tabular}{rccccc}
\toprule
$N$ & HLD & MMD & RootLL & TDI & Typicality \\
\midrule
1 & $0.078 \pm 0.046$ & $0.049 \pm 0.029$ & $0.080 \pm 0.034$ & $0.080 \pm 0.034$ & $0.075 \pm 0.047$ \\
10 & $0.089 \pm 0.007$ & $0.103 \pm 0.005$ & $0.093 \pm 0.035$ & $0.092 \pm 0.035$ & $0.086 \pm 0.053$ \\
50 & $0.097 \pm 0.010$ & $0.121 \pm 0.010$ & $0.100 \pm 0.019$ & $0.100 \pm 0.021$ & $0.089 \pm 0.057$ \\
100 & $0.128 \pm 0.018$ & $0.140 \pm 0.016$ & $0.102 \pm 0.031$ & $0.100 \pm 0.031$ & $0.104 \pm 0.079$ \\
200 & $0.138 \pm 0.026$ & $0.184 \pm 0.018$ & $0.171 \pm 0.054$ & $0.168 \pm 0.060$ & $0.217 \pm 0.109$ \\
500 & $0.189 \pm 0.046$ & $0.414 \pm 0.092$ & $0.292 \pm 0.189$ & $0.288 \pm 0.198$ & $0.451 \pm 0.354$ \\
1000 & $0.279 \pm 0.095$ & $0.661 \pm 0.119$ & $0.417 \pm 0.313$ & $0.406 \pm 0.323$ & $0.533 \pm 0.376$ \\
1500 & $0.418 \pm 0.169$ & $0.653 \pm 0.126$ & $0.501 \pm 0.361$ & $0.484 \pm 0.375$ & $0.582 \pm 0.390$ \\
2000 & $0.517 \pm 0.228$ & $0.643 \pm 0.116$ & $0.545 \pm 0.369$ & $0.517 \pm 0.396$ & $0.635 \pm 0.375$ \\
2500 & $0.621 \pm 0.268$ & $0.664 \pm 0.118$ & $0.673 \pm 0.381$ & $0.640 \pm 0.407$ & $0.659 \pm 0.374$ \\
\bottomrule
\end{tabular}
\caption{\(\mathrm{FPR}_{\mathrm{model}}\), averaged over 5 tabular datasets, $\alpha=0.1$.}
\label{tab:fpr2_a01}
\end{table}

\section{Results}

\subsection{Tabular}
Tables~\ref{tab:census_kdd},~\ref{tab:connect4},~\ref{tab:covertype}, and ~\ref{tab:sensorless} show the detection results across the various Tabular datasets. Across the datasets, Table~(a) reports the model-level false positive rate under the PC-vs-PC null, where all detectors are near-nominal throughout. Table~(b) replaces the second PC batch with real held-out ID data, exposing model misspecification: as $T$ grows, every detector's FPR inflates above $\alpha$ because the test becomes powerful enough to detect the residual gap between the circuit and the true data distribution. Crucially, HLD is by far the most robust detector in this regime, on Sensorless it holds the \(\mathrm{FPR}_{\mathrm{data}}\) down to $.351$ at $T=1000$, whereas MMD saturates completely at $1.00$ and typicality climbs to $.741$. This gap widens with batch size across every dataset, demonstrating that the hierarchical likelihood representation degrades gracefully where competing statistics collapse, and that HLD is the method to remain stable at large $T$ under PC-vs-real shift. Table~(c) shows detection rate against genuine OOD batches, where HLD is the strongest performer, reaching perfect or near-perfect detection by $T{=}50$ to $100$ on every dataset. Its advantage is sharpest exactly where the baselines fail: on Connect-4, RootLL, typicality, and TDI plateau near $.34$, while HLD drives detection to $1.00$.
\subsection{MNIST $28\times28$}
Table~\ref{tab:mnist} reports the full $28\times28$ resolution detection results discussed in the main paper. The elevated $\mathrm{FPR}_{\mathrm{data}}$ seen here is explained directly by the training analysis in Section C.2 as the circuit samples in Figure~\ref{fig:MD} show, the HCLT does not
learn the full-resolution pixel distribution well, and this poor fit is what
drives the false positives on real ID data at $28\times28$.

\subsection{Image}
Tables~\ref{tab:mnist-0-vs-1} to \ref{tab:mnist-8-vs-9}, give the full per-pair
breakdown of the averaged MNIST results reported in the main paper, confirming
that the aggregate trends hold pair-by-pair rather than being driven by a few
favorable pairs. Under the PC-vs-PC null ($\mathrm{FPR}_{\mathrm{model}}$),
all five methods are well calibrated across pairs. Detection power is similarly
uniform: although HLD is not the strongest detector at $T{=}1$, it is best or
tied-best on every pair from $T{=}10$ onward, reaching near-perfect detection. The methods separate most clearly on real held-out ID data: HLD attains
the lowest false-positive rate ($\mathrm{FPR}_{\mathrm{data}}$) on the large
majority of pairs at $T{=}1000$, and even its worst pair ($3$-vs-$7$) remains
comparable to the best-performing baseline, so the slow growth of
$\mathrm{FPR}_{\mathrm{data}}$ is systematic across the pair set rather than an
artifact of averaging.
\twocolumn
\begin{table}[t]
\centering
\small
\setlength{\tabcolsep}{1mm}
\begin{tabular}{rccccc}
\toprule
$T$ & HLD & MMD & RootLL & typicality & TDI \\
\midrule
 \multicolumn{6}{l}{\emph{(a) \(\mathrm{FPR}_{\mathrm{model}}\) --- PC-vs-PC null (target $0.05$)}} \\
1    & $.009{\pm}.00$ & $.038{\pm}.03$ & $.053{\pm}.01$ & $.051{\pm}.01$ & $.053{\pm}.01$ \\
10   & $.070{\pm}.01$ & $.053{\pm}.00$ & $.054{\pm}.02$ & $.049{\pm}.00$ & $.053{\pm}.01$ \\
50   & $.072{\pm}.00$ & $.042{\pm}.01$ & $.050{\pm}.00$ & $.048{\pm}.02$ & $.049{\pm}.00$ \\
100  & $.051{\pm}.00$ & $.040{\pm}.01$ & $.047{\pm}.02$ & $.047{\pm}.00$ & $.047{\pm}.02$ \\
200  & $.067{\pm}.01$ & $.060{\pm}.01$ & $.064{\pm}.02$ & $.052{\pm}.01$ & $.064{\pm}.02$ \\
500  & $.056{\pm}.01$ & $.050{\pm}.01$ & $.046{\pm}.01$ & $.046{\pm}.02$ & $.047{\pm}.01$ \\
\bottomrule
\end{tabular}
\caption{Detection results on Census-KDD , $\alpha=0.05$.}
\label{tab:census_kdd}
\end{table}

\begin{table}[t]
\centering
\small
\setlength{\tabcolsep}{1mm}
\begin{tabular}{rccccc}
\toprule
$T$ & HLD & MMD & RootLL & typicality & TDI \\
\midrule
 \multicolumn{6}{l}{\emph{(a) \(\mathrm{FPR}_{\mathrm{model}}\) --- PC-vs-PC null (target $0.05$)}} \\
1    & $.128{\pm}.03$ & $.031{\pm}.01$ & $.045{\pm}.01$ & $.039{\pm}.00$ & $.042{\pm}.01$ \\
10   & $.065{\pm}.00$ & $.041{\pm}.01$ & $.045{\pm}.01$ & $.045{\pm}.01$ & $.041{\pm}.01$ \\
50   & $.073{\pm}.00$ & $.051{\pm}.00$ & $.053{\pm}.00$ & $.056{\pm}.02$ & $.054{\pm}.01$ \\
100  & $.071{\pm}.01$ & $.041{\pm}.01$ & $.043{\pm}.00$ & $.043{\pm}.01$ & $.045{\pm}.00$ \\
200  & $.061{\pm}.00$ & $.052{\pm}.00$ & $.055{\pm}.00$ & $.043{\pm}.02$ & $.055{\pm}.00$ \\
500  & $.063{\pm}.01$ & $.046{\pm}.00$ & $.053{\pm}.02$ & $.050{\pm}.01$ & $.051{\pm}.02$ \\
\bottomrule
\end{tabular}
\caption{Detection results on Connect-4 , $\alpha=0.05$.}
\label{tab:connect4}
\end{table}

\begin{table}[t]
\centering
\small
\setlength{\tabcolsep}{1mm}
\begin{tabular}{rccccc}
\toprule
$T$ & HLD & MMD & RootLL & typicality & TDI \\
\midrule
 \multicolumn{6}{l}{\emph{(a) \(\mathrm{FPR}_{\mathrm{model}}\) --- PC-vs-PC null (target $0.05$)}} \\
1    & $.048{\pm}.03$ & $.000{\pm}.00$ & $.041{\pm}.01$ & $.056{\pm}.01$ & $.043{\pm}.01$ \\
10   & $.050{\pm}.01$ & $.044{\pm}.01$ & $.047{\pm}.01$ & $.045{\pm}.02$ & $.046{\pm}.01$ \\
50   & $.066{\pm}.01$ & $.051{\pm}.01$ & $.056{\pm}.01$ & $.053{\pm}.01$ & $.057{\pm}.01$ \\
100  & $.077{\pm}.01$ & $.052{\pm}.02$ & $.046{\pm}.01$ & $.052{\pm}.01$ & $.048{\pm}.01$ \\
200  & $.062{\pm}.01$ & $.053{\pm}.01$ & $.054{\pm}.01$ & $.057{\pm}.02$ & $.058{\pm}.01$ \\
500  & $.066{\pm}.02$ & $.050{\pm}.01$ & $.051{\pm}.01$ & $.045{\pm}.01$ & $.051{\pm}.01$ \\
\bottomrule
\end{tabular}
\caption{Detection results on Covertype , $\alpha=0.05$.}
\label{tab:covertype}
\end{table}

\begin{table}[t]
\centering
\small
\setlength{\tabcolsep}{1mm}
\begin{tabular}{rccccc}
\toprule
$T$ & HLD & MMD & RootLL & typicality & TDI \\
\midrule
 \multicolumn{6}{l}{\emph{(a) \(\mathrm{FPR}_{\mathrm{model}}\) --- PC-vs-PC null (target $0.05$)}} \\
1    & $.044{\pm}.02$ & $.025{\pm}.02$ & $.047{\pm}.01$ & $.051{\pm}.01$ & $.047{\pm}.01$ \\
10   & $.062{\pm}.01$ & $.050{\pm}.01$ & $.058{\pm}.01$ & $.053{\pm}.01$ & $.059{\pm}.01$ \\
50   & $.060{\pm}.01$ & $.048{\pm}.01$ & $.049{\pm}.01$ & $.045{\pm}.01$ & $.048{\pm}.01$ \\
100  & $.061{\pm}.01$ & $.050{\pm}.01$ & $.047{\pm}.01$ & $.054{\pm}.02$ & $.047{\pm}.01$ \\
200  & $.067{\pm}.01$ & $.049{\pm}.01$ & $.040{\pm}.01$ & $.048{\pm}.02$ & $.040{\pm}.01$ \\
500  & $.070{\pm}.01$ & $.052{\pm}.02$ & $.053{\pm}.02$ & $.046{\pm}.01$ & $.054{\pm}.02$ \\
\bottomrule
\end{tabular}
\caption{Detection results on Sensorless , $\alpha=0.05$.}
\label{tab:sensorless}
\end{table}
\begin{table}[t]
\centering
\small
\setlength{\tabcolsep}{1mm}
\begin{tabular}{rccccc}
\toprule
$T$ & HLD & MMD & RootLL & typicality & TDI \\
\midrule
 \multicolumn{5}{l}{\emph{(a) \(\mathrm{FPR}_{\mathrm{model}}\) --- PC-vs-PC null (target $0.05$)}} \\
1    & $.149{\pm}.03$ & $.038{\pm}.01$ & $.050{\pm}.01$ & $.056{\pm}.01$ & $.055{\pm}.01$ \\
10   & $.077{\pm}.01$ & $.052{\pm}.01$ & $.051{\pm}.01$ & $.049{\pm}.01$&  $.050{\pm}.01$\\
50   & $.070{\pm}.01$ & $.050{\pm}.01$ & $.051{\pm}.01$ & $.049{\pm}.01$ & $.053{\pm}.01$  \\
100  & $.069{\pm}.01$ & $.050{\pm}.01$ & $.049{\pm}.01$ & $.049{\pm}.01$ & $.051{\pm}.01$  \\
200  & $.069{\pm}.01$ & $.045{\pm}.01$ & $.053{\pm}.01$ & $.055{\pm}.01$ & $.052{\pm}.01$  \\
500  & $.066{\pm}.01$ & $.047{\pm}.01$ & $.050{\pm}.01$ & $.057{\pm}.01$ & $.052{\pm}.02$ \\
\bottomrule
\end{tabular}
\caption{Detection results on MNIST ($28\times28$) $\alpha=0.05$.}
\label{tab:mnist}
\end{table}

\begin{table}[t]
\centering
\small
\setlength{\tabcolsep}{1mm}
\begin{tabular}{rccccc}
\toprule
$T$ & HLD & MMD & RootLL & Typicality & TDI \\
\midrule
\multicolumn{6}{l}{\emph{(a) \(\mathrm{FPR}_{\mathrm model}\) --- PC-vs-PC null (target $0.05$)}} \\
1 & $.036{\pm}.01$ & $.032{\pm}.01$ & $.046{\pm}.01$ & $.050{\pm}.01$ & $.044{\pm}.01$ \\
10 & $.042{\pm}.01$ & $.076{\pm}.01$ & $.042{\pm}.01$ & $.028{\pm}.01$ & $.044{\pm}.01$ \\
50 & $.088{\pm}.01$ & $.052{\pm}.01$ & $.050{\pm}.01$ & $.060{\pm}.01$ & $.052{\pm}.01$ \\
100 & $.066{\pm}.01$ & $.030{\pm}.01$ & $.040{\pm}.01$ & $.058{\pm}.01$ & $.038{\pm}.01$ \\
200 & $.072{\pm}.01$ & $.074{\pm}.01$ & $.042{\pm}.01$ & $.064{\pm}.01$ & $.046{\pm}.01$ \\
500 & $.090{\pm}.01$ & $.040{\pm}.01$ & $.060{\pm}.01$ & $.076{\pm}.01$ & $.064{\pm}.01$ \\
\bottomrule
\end{tabular}
\caption{Detection results on MNIST, pair 0-vs-1, $\alpha=0.05$.}
\label{tab:mnist-0-vs-1}
\end{table}

\begin{table}[t]
\centering
\small
\setlength{\tabcolsep}{1mm}
\begin{tabular}{rccccc}
\toprule
$T$ & HLD & MMD & RootLL & Typicality & TDI \\
\midrule
\multicolumn{6}{l}{\emph{(a) \(\mathrm{FPR}_{\mathrm model}\) --- PC-vs-PC null (target $0.05$)}} \\
1 & $.032{\pm}.01$ & $.042{\pm}.01$ & $.034{\pm}.01$ & $.060{\pm}.01$ & $.034{\pm}.01$ \\
10 & $.062{\pm}.01$ & $.032{\pm}.01$ & $.042{\pm}.01$ & $.046{\pm}.01$ & $.042{\pm}.01$ \\
50 & $.048{\pm}.01$ & $.058{\pm}.01$ & $.042{\pm}.01$ & $.042{\pm}.01$ & $.042{\pm}.01$ \\
100 & $.084{\pm}.01$ & $.060{\pm}.01$ & $.022{\pm}.01$ & $.028{\pm}.01$ & $.022{\pm}.01$ \\
200 & $.074{\pm}.01$ & $.064{\pm}.01$ & $.056{\pm}.01$ & $.054{\pm}.01$ & $.060{\pm}.01$ \\
500 & $.058{\pm}.01$ & $.042{\pm}.01$ & $.060{\pm}.01$ & $.062{\pm}.01$ & $.062{\pm}.01$ \\
\bottomrule
\end{tabular}
\caption{Detection results on MNIST, pair 0-vs-2, $\alpha=0.05$.}
\label{tab:mnist-0-vs-2}
\end{table}

\begin{table}[t]
\centering
\small
\setlength{\tabcolsep}{1mm}
\begin{tabular}{rccccc}
\toprule
$T$ & HLD & MMD & RootLL & Typicality & TDI \\
\midrule
\multicolumn{6}{l}{\emph{(a) \(\mathrm{FPR}_{\mathrm model}\) --- PC-vs-PC null (target $0.05$)}} \\
1 & $.034{\pm}.01$ & $.038{\pm}.01$ & $.028{\pm}.01$ & $.038{\pm}.01$ & $.026{\pm}.01$ \\
10 & $.066{\pm}.01$ & $.054{\pm}.01$ & $.048{\pm}.01$ & $.054{\pm}.01$ & $.046{\pm}.01$ \\
50 & $.068{\pm}.01$ & $.056{\pm}.01$ & $.066{\pm}.01$ & $.038{\pm}.01$ & $.066{\pm}.01$ \\
100 & $.072{\pm}.01$ & $.042{\pm}.01$ & $.050{\pm}.01$ & $.058{\pm}.01$ & $.050{\pm}.01$ \\
200 & $.074{\pm}.01$ & $.052{\pm}.01$ & $.038{\pm}.01$ & $.044{\pm}.01$ & $.042{\pm}.01$ \\
500 & $.070{\pm}.01$ & $.040{\pm}.01$ & $.064{\pm}.01$ & $.034{\pm}.01$ & $.064{\pm}.01$ \\
\bottomrule
\end{tabular}
\caption{Detection results on MNIST, pair 0-vs-3, $\alpha=0.05$.}
\label{tab:mnist-0-vs-3}
\end{table}

\begin{table}[t]
\centering
\small
\setlength{\tabcolsep}{1mm}
\begin{tabular}{rccccc}
\toprule
$T$ & HLD & MMD & RootLL & Typicality & TDI \\
\midrule
\multicolumn{6}{l}{\emph{(a) \(\mathrm{FPR}_{\mathrm model}\) --- PC-vs-PC null (target $0.05$)}} \\
1 & $.024{\pm}.01$ & $.036{\pm}.01$ & $.028{\pm}.01$ & $.036{\pm}.01$ & $.028{\pm}.01$ \\
10 & $.056{\pm}.01$ & $.062{\pm}.01$ & $.048{\pm}.01$ & $.046{\pm}.01$ & $.046{\pm}.01$ \\
50 & $.084{\pm}.01$ & $.068{\pm}.01$ & $.052{\pm}.01$ & $.026{\pm}.01$ & $.052{\pm}.01$ \\
100 & $.074{\pm}.01$ & $.040{\pm}.01$ & $.040{\pm}.01$ & $.046{\pm}.01$ & $.040{\pm}.01$ \\
200 & $.064{\pm}.01$ & $.046{\pm}.01$ & $.064{\pm}.01$ & $.058{\pm}.01$ & $.066{\pm}.01$ \\
500 & $.082{\pm}.01$ & $.038{\pm}.01$ & $.058{\pm}.01$ & $.056{\pm}.01$ & $.054{\pm}.01$ \\
\bottomrule
\end{tabular}
\caption{Detection results on MNIST, pair 0-vs-4, $\alpha=0.05$.}
\label{tab:mnist-0-vs-4}
\end{table}

\begin{table}[t]
\centering
\small
\setlength{\tabcolsep}{1mm}
\begin{tabular}{rccccc}
\toprule
$T$ & HLD & MMD & RootLL & Typicality & TDI \\
\midrule
\multicolumn{6}{l}{\emph{(a) \(\mathrm{FPR}_{\mathrm model}\) --- PC-vs-PC null (target $0.05$)}} \\
1 & $.016{\pm}.01$ & $.046{\pm}.01$ & $.032{\pm}.01$ & $.044{\pm}.01$ & $.034{\pm}.01$ \\
10 & $.048{\pm}.01$ & $.050{\pm}.01$ & $.034{\pm}.01$ & $.016{\pm}.01$ & $.032{\pm}.01$ \\
50 & $.072{\pm}.01$ & $.060{\pm}.01$ & $.056{\pm}.01$ & $.044{\pm}.01$ & $.054{\pm}.01$ \\
100 & $.060{\pm}.01$ & $.048{\pm}.01$ & $.064{\pm}.01$ & $.056{\pm}.01$ & $.064{\pm}.01$ \\
200 & $.064{\pm}.01$ & $.042{\pm}.01$ & $.050{\pm}.01$ & $.048{\pm}.01$ & $.048{\pm}.01$ \\
500 & $.066{\pm}.01$ & $.044{\pm}.01$ & $.052{\pm}.01$ & $.040{\pm}.01$ & $.054{\pm}.01$ \\
\bottomrule
\end{tabular}
\caption{Detection results on MNIST, pair 0-vs-5, $\alpha=0.05$.}
\label{tab:mnist-0-vs-5}
\end{table}

\begin{table}[t]
\centering
\small
\setlength{\tabcolsep}{1mm}
\begin{tabular}{rccccc}
\toprule
$T$ & HLD & MMD & RootLL & Typicality & TDI \\
\midrule
\multicolumn{6}{l}{\emph{(a) \(\mathrm{FPR}_{\mathrm model}\) --- PC-vs-PC null (target $0.05$)}} \\
1 & $.036{\pm}.01$ & $.060{\pm}.01$ & $.036{\pm}.01$ & $.038{\pm}.01$ & $.036{\pm}.01$ \\
10 & $.060{\pm}.01$ & $.058{\pm}.01$ & $.044{\pm}.01$ & $.052{\pm}.01$ & $.046{\pm}.01$ \\
50 & $.060{\pm}.01$ & $.048{\pm}.01$ & $.046{\pm}.01$ & $.042{\pm}.01$ & $.044{\pm}.01$ \\
100 & $.066{\pm}.01$ & $.052{\pm}.01$ & $.038{\pm}.01$ & $.040{\pm}.01$ & $.038{\pm}.01$ \\
200 & $.060{\pm}.01$ & $.062{\pm}.01$ & $.040{\pm}.01$ & $.048{\pm}.01$ & $.042{\pm}.01$ \\
500 & $.066{\pm}.01$ & $.040{\pm}.01$ & $.058{\pm}.01$ & $.040{\pm}.01$ & $.060{\pm}.01$ \\
\bottomrule
\end{tabular}
\caption{Detection results on MNIST, pair 0-vs-6, $\alpha=0.05$.}
\label{tab:mnist-0-vs-6}
\end{table}

\begin{table}[t]
\centering
\small
\setlength{\tabcolsep}{1mm}
\begin{tabular}{rccccc}
\toprule
$T$ & HLD & MMD & RootLL & Typicality & TDI \\
\midrule
\multicolumn{6}{l}{\emph{(a) \(\mathrm{FPR}_{\mathrm model}\) --- PC-vs-PC null (target $0.05$)}} \\
1 & $.026{\pm}.01$ & $.038{\pm}.01$ & $.044{\pm}.01$ & $.054{\pm}.01$ & $.046{\pm}.01$ \\
10 & $.078{\pm}.01$ & $.062{\pm}.01$ & $.032{\pm}.01$ & $.046{\pm}.01$ & $.034{\pm}.01$ \\
50 & $.048{\pm}.01$ & $.050{\pm}.01$ & $.080{\pm}.01$ & $.046{\pm}.01$ & $.084{\pm}.01$ \\
100 & $.082{\pm}.01$ & $.040{\pm}.01$ & $.030{\pm}.01$ & $.046{\pm}.01$ & $.026{\pm}.01$ \\
200 & $.070{\pm}.01$ & $.044{\pm}.01$ & $.056{\pm}.01$ & $.052{\pm}.01$ & $.060{\pm}.01$ \\
500 & $.080{\pm}.01$ & $.042{\pm}.01$ & $.074{\pm}.01$ & $.044{\pm}.01$ & $.076{\pm}.01$ \\
\bottomrule
\end{tabular}
\caption{Detection results on MNIST, pair 0-vs-7, $\alpha=0.05$.}
\label{tab:mnist-0-vs-7}
\end{table}

\begin{table}[t]
\centering
\small
\setlength{\tabcolsep}{1mm}
\begin{tabular}{rccccc}
\toprule
$T$ & HLD & MMD & RootLL & Typicality & TDI \\
\midrule
\multicolumn{6}{l}{\emph{(a) \(\mathrm{FPR}_{\mathrm model}\) --- PC-vs-PC null (target $0.05$)}} \\
1 & $.028{\pm}.01$ & $.044{\pm}.01$ & $.028{\pm}.01$ & $.034{\pm}.01$ & $.032{\pm}.01$ \\
10 & $.060{\pm}.01$ & $.080{\pm}.01$ & $.054{\pm}.01$ & $.054{\pm}.01$ & $.056{\pm}.01$ \\
50 & $.066{\pm}.01$ & $.046{\pm}.01$ & $.046{\pm}.01$ & $.060{\pm}.01$ & $.046{\pm}.01$ \\
100 & $.092{\pm}.01$ & $.042{\pm}.01$ & $.042{\pm}.01$ & $.054{\pm}.01$ & $.044{\pm}.01$ \\
200 & $.062{\pm}.01$ & $.042{\pm}.01$ & $.046{\pm}.01$ & $.050{\pm}.01$ & $.048{\pm}.01$ \\
500 & $.070{\pm}.01$ & $.048{\pm}.01$ & $.076{\pm}.01$ & $.058{\pm}.01$ & $.080{\pm}.01$ \\
\bottomrule
\end{tabular}
\caption{Detection results on MNIST, pair 0-vs-8, $\alpha=0.05$.}
\label{tab:mnist-0-vs-8}
\end{table}

\begin{table}[t]
\centering
\small
\setlength{\tabcolsep}{1mm}
\begin{tabular}{rccccc}
\toprule
$T$ & HLD & MMD & RootLL & Typicality & TDI \\
\midrule
\multicolumn{6}{l}{\emph{(a) \(\mathrm{FPR}_{\mathrm model}\) --- PC-vs-PC null (target $0.05$)}} \\
1 & $.032{\pm}.01$ & $.044{\pm}.01$ & $.040{\pm}.01$ & $.056{\pm}.01$ & $.042{\pm}.01$ \\
10 & $.066{\pm}.01$ & $.060{\pm}.01$ & $.028{\pm}.01$ & $.044{\pm}.01$ & $.032{\pm}.01$ \\
50 & $.076{\pm}.01$ & $.060{\pm}.01$ & $.058{\pm}.01$ & $.056{\pm}.01$ & $.062{\pm}.01$ \\
100 & $.082{\pm}.01$ & $.054{\pm}.01$ & $.042{\pm}.01$ & $.054{\pm}.01$ & $.052{\pm}.01$ \\
200 & $.090{\pm}.01$ & $.062{\pm}.01$ & $.044{\pm}.01$ & $.036{\pm}.01$ & $.044{\pm}.01$ \\
500 & $.096{\pm}.01$ & $.054{\pm}.01$ & $.034{\pm}.01$ & $.028{\pm}.01$ & $.040{\pm}.01$ \\
\bottomrule
\end{tabular}
\caption{Detection results on MNIST, pair 0-vs-9, $\alpha=0.05$.}
\label{tab:mnist-0-vs-9}
\end{table}

\begin{table}[t]
\centering
\small
\setlength{\tabcolsep}{1mm}
\begin{tabular}{rccccc}
\toprule
$T$ & HLD & MMD & RootLL & Typicality & TDI \\
\midrule
\multicolumn{6}{l}{\emph{(a) \(\mathrm{FPR}_{\mathrm model}\) --- PC-vs-PC null (target $0.05$)}} \\
1 & $.112{\pm}.01$ & $.048{\pm}.01$ & $.060{\pm}.01$ & $.036{\pm}.01$ & $.064{\pm}.01$ \\
10 & $.082{\pm}.01$ & $.080{\pm}.01$ & $.038{\pm}.01$ & $.058{\pm}.01$ & $.038{\pm}.01$ \\
50 & $.072{\pm}.01$ & $.060{\pm}.01$ & $.038{\pm}.01$ & $.034{\pm}.01$ & $.036{\pm}.01$ \\
100 & $.064{\pm}.01$ & $.038{\pm}.01$ & $.038{\pm}.01$ & $.034{\pm}.01$ & $.038{\pm}.01$ \\
200 & $.072{\pm}.01$ & $.042{\pm}.01$ & $.040{\pm}.01$ & $.032{\pm}.01$ & $.046{\pm}.01$ \\
500 & $.070{\pm}.01$ & $.040{\pm}.01$ & $.042{\pm}.01$ & $.036{\pm}.01$ & $.044{\pm}.01$ \\
\bottomrule
\end{tabular}
\caption{Detection results on MNIST, pair 1-vs-2, $\alpha=0.05$.}
\label{tab:mnist-1-vs-2}
\end{table}

\begin{table}[t]
\centering
\small
\setlength{\tabcolsep}{1mm}
\begin{tabular}{rccccc}
\toprule
$T$ & HLD & MMD & RootLL & Typicality & TDI \\
\midrule
\multicolumn{6}{l}{\emph{(a) \(\mathrm{FPR}_{\mathrm model}\) --- PC-vs-PC null (target $0.05$)}} \\
1 & $.106{\pm}.01$ & $.042{\pm}.01$ & $.058{\pm}.01$ & $.048{\pm}.01$ & $.060{\pm}.01$ \\
10 & $.048{\pm}.01$ & $.044{\pm}.01$ & $.038{\pm}.01$ & $.060{\pm}.01$ & $.040{\pm}.01$ \\
50 & $.066{\pm}.01$ & $.026{\pm}.01$ & $.046{\pm}.01$ & $.038{\pm}.01$ & $.046{\pm}.01$ \\
100 & $.052{\pm}.01$ & $.050{\pm}.01$ & $.034{\pm}.01$ & $.046{\pm}.01$ & $.034{\pm}.01$ \\
200 & $.070{\pm}.01$ & $.048{\pm}.01$ & $.070{\pm}.01$ & $.052{\pm}.01$ & $.070{\pm}.01$ \\
500 & $.072{\pm}.01$ & $.040{\pm}.01$ & $.040{\pm}.01$ & $.052{\pm}.01$ & $.044{\pm}.01$ \\
\bottomrule
\end{tabular}
\caption{Detection results on MNIST, pair 1-vs-3, $\alpha=0.05$.}
\label{tab:mnist-1-vs-3}
\end{table}

\begin{table}[t]
\centering
\small
\setlength{\tabcolsep}{1mm}
\begin{tabular}{rccccc}
\toprule
$T$ & HLD & MMD & RootLL & Typicality & TDI \\
\midrule
\multicolumn{6}{l}{\emph{(a) \(\mathrm{FPR}_{\mathrm model}\) --- PC-vs-PC null (target $0.05$)}} \\
1 & $.102{\pm}.01$ & $.040{\pm}.01$ & $.058{\pm}.01$ & $.040{\pm}.01$ & $.056{\pm}.01$ \\
10 & $.056{\pm}.01$ & $.056{\pm}.01$ & $.040{\pm}.01$ & $.044{\pm}.01$ & $.042{\pm}.01$ \\
50 & $.064{\pm}.01$ & $.066{\pm}.01$ & $.050{\pm}.01$ & $.046{\pm}.01$ & $.040{\pm}.01$ \\
100 & $.068{\pm}.01$ & $.044{\pm}.01$ & $.032{\pm}.01$ & $.036{\pm}.01$ & $.032{\pm}.01$ \\
200 & $.058{\pm}.01$ & $.046{\pm}.01$ & $.036{\pm}.01$ & $.052{\pm}.01$ & $.036{\pm}.01$ \\
500 & $.080{\pm}.01$ & $.028{\pm}.01$ & $.054{\pm}.01$ & $.056{\pm}.01$ & $.056{\pm}.01$ \\
\bottomrule
\end{tabular}
\caption{Detection results on MNIST, pair 1-vs-4, $\alpha=0.05$.}
\label{tab:mnist-1-vs-4}
\end{table}

\begin{table}[t]
\centering
\small
\setlength{\tabcolsep}{1mm}
\begin{tabular}{rccccc}
\toprule
$T$ & HLD & MMD & RootLL & Typicality & TDI \\
\midrule
\multicolumn{6}{l}{\emph{(a) \(\mathrm{FPR}_{\mathrm model}\) --- PC-vs-PC null (target $0.05$)}} \\
1 & $.110{\pm}.01$ & $.032{\pm}.01$ & $.044{\pm}.01$ & $.034{\pm}.01$ & $.046{\pm}.01$ \\
10 & $.080{\pm}.01$ & $.054{\pm}.01$ & $.042{\pm}.01$ & $.054{\pm}.01$ & $.050{\pm}.01$ \\
50 & $.072{\pm}.01$ & $.062{\pm}.01$ & $.050{\pm}.01$ & $.050{\pm}.01$ & $.048{\pm}.01$ \\
100 & $.050{\pm}.01$ & $.034{\pm}.01$ & $.030{\pm}.01$ & $.042{\pm}.01$ & $.028{\pm}.01$ \\
200 & $.056{\pm}.01$ & $.042{\pm}.01$ & $.048{\pm}.01$ & $.048{\pm}.01$ & $.050{\pm}.01$ \\
500 & $.052{\pm}.01$ & $.034{\pm}.01$ & $.038{\pm}.01$ & $.024{\pm}.01$ & $.038{\pm}.01$ \\
\bottomrule
\end{tabular}
\caption{Detection results on MNIST, pair 1-vs-5, $\alpha=0.05$.}
\label{tab:mnist-1-vs-5}
\end{table}

\begin{table}[t]
\centering
\small
\setlength{\tabcolsep}{1mm}
\begin{tabular}{rccccc}
\toprule
$T$ & HLD & MMD & RootLL & Typicality & TDI \\
\midrule
\multicolumn{6}{l}{\emph{(a) \(\mathrm{FPR}_{\mathrm model}\) --- PC-vs-PC null (target $0.05$)}} \\
1 & $.120{\pm}.01$ & $.050{\pm}.01$ & $.064{\pm}.01$ & $.048{\pm}.01$ & $.062{\pm}.01$ \\
10 & $.058{\pm}.01$ & $.048{\pm}.01$ & $.038{\pm}.01$ & $.058{\pm}.01$ & $.040{\pm}.01$ \\
50 & $.086{\pm}.01$ & $.068{\pm}.01$ & $.062{\pm}.01$ & $.050{\pm}.01$ & $.060{\pm}.01$ \\
100 & $.076{\pm}.01$ & $.058{\pm}.01$ & $.040{\pm}.01$ & $.036{\pm}.01$ & $.040{\pm}.01$ \\
200 & $.058{\pm}.01$ & $.040{\pm}.01$ & $.042{\pm}.01$ & $.042{\pm}.01$ & $.042{\pm}.01$ \\
500 & $.046{\pm}.01$ & $.030{\pm}.01$ & $.036{\pm}.01$ & $.038{\pm}.01$ & $.038{\pm}.01$ \\
\bottomrule
\end{tabular}
\caption{Detection results on MNIST, pair 1-vs-6, $\alpha=0.05$.}
\label{tab:mnist-1-vs-6}
\end{table}

\begin{table}[t]
\centering
\small
\setlength{\tabcolsep}{1mm}
\begin{tabular}{rccccc}
\toprule
$T$ & HLD & MMD & RootLL & Typicality & TDI \\
\midrule
\multicolumn{6}{l}{\emph{(a) \(\mathrm{FPR}_{\mathrm model}\) --- PC-vs-PC null (target $0.05$)}} \\
1 & $.098{\pm}.01$ & $.032{\pm}.01$ & $.062{\pm}.01$ & $.050{\pm}.01$ & $.060{\pm}.01$ \\
10 & $.080{\pm}.01$ & $.048{\pm}.01$ & $.044{\pm}.01$ & $.046{\pm}.01$ & $.044{\pm}.01$ \\
50 & $.078{\pm}.01$ & $.050{\pm}.01$ & $.046{\pm}.01$ & $.046{\pm}.01$ & $.046{\pm}.01$ \\
100 & $.066{\pm}.01$ & $.038{\pm}.01$ & $.046{\pm}.01$ & $.054{\pm}.01$ & $.044{\pm}.01$ \\
200 & $.076{\pm}.01$ & $.038{\pm}.01$ & $.048{\pm}.01$ & $.062{\pm}.01$ & $.048{\pm}.01$ \\
500 & $.050{\pm}.01$ & $.028{\pm}.01$ & $.040{\pm}.01$ & $.040{\pm}.01$ & $.040{\pm}.01$ \\
\bottomrule
\end{tabular}
\caption{Detection results on MNIST, pair 1-vs-7, $\alpha=0.05$.}
\label{tab:mnist-1-vs-7}
\end{table}

\begin{table}[t]
\centering
\small
\setlength{\tabcolsep}{1mm}
\begin{tabular}{rccccc}
\toprule
$T$ & HLD & MMD & RootLL & Typicality & TDI \\
\midrule
\multicolumn{6}{l}{\emph{(a) \(\mathrm{FPR}_{\mathrm model}\) --- PC-vs-PC null (target $0.05$)}} \\
1 & $.126{\pm}.01$ & $.038{\pm}.01$ & $.052{\pm}.01$ & $.044{\pm}.01$ & $.052{\pm}.01$ \\
10 & $.066{\pm}.01$ & $.042{\pm}.01$ & $.046{\pm}.01$ & $.050{\pm}.01$ & $.048{\pm}.01$ \\
50 & $.056{\pm}.01$ & $.042{\pm}.01$ & $.044{\pm}.01$ & $.046{\pm}.01$ & $.038{\pm}.01$ \\
100 & $.084{\pm}.01$ & $.034{\pm}.01$ & $.036{\pm}.01$ & $.046{\pm}.01$ & $.036{\pm}.01$ \\
200 & $.078{\pm}.01$ & $.040{\pm}.01$ & $.044{\pm}.01$ & $.032{\pm}.01$ & $.044{\pm}.01$ \\
500 & $.052{\pm}.01$ & $.036{\pm}.01$ & $.062{\pm}.01$ & $.046{\pm}.01$ & $.062{\pm}.01$ \\
\bottomrule
\end{tabular}
\caption{Detection results on MNIST, pair 1-vs-8, $\alpha=0.05$.}
\label{tab:mnist-1-vs-8}
\end{table}

\begin{table}[t]
\centering
\small
\setlength{\tabcolsep}{1mm}
\begin{tabular}{rccccc}
\toprule
$T$ & HLD & MMD & RootLL & Typicality & TDI \\
\midrule
\multicolumn{6}{l}{\emph{(a) \(\mathrm{FPR}_{\mathrm model}\) --- PC-vs-PC null (target $0.05$)}} \\
1 & $.116{\pm}.01$ & $.036{\pm}.01$ & $.036{\pm}.01$ & $.034{\pm}.01$ & $.036{\pm}.01$ \\
10 & $.084{\pm}.01$ & $.052{\pm}.01$ & $.030{\pm}.01$ & $.056{\pm}.01$ & $.032{\pm}.01$ \\
50 & $.046{\pm}.01$ & $.064{\pm}.01$ & $.066{\pm}.01$ & $.056{\pm}.01$ & $.060{\pm}.01$ \\
100 & $.052{\pm}.01$ & $.042{\pm}.01$ & $.044{\pm}.01$ & $.058{\pm}.01$ & $.044{\pm}.01$ \\
200 & $.064{\pm}.01$ & $.046{\pm}.01$ & $.022{\pm}.01$ & $.046{\pm}.01$ & $.026{\pm}.01$ \\
500 & $.084{\pm}.01$ & $.048{\pm}.01$ & $.048{\pm}.01$ & $.040{\pm}.01$ & $.050{\pm}.01$ \\
\bottomrule
\end{tabular}
\caption{Detection results on MNIST, pair 1-vs-9, $\alpha=0.05$.}
\label{tab:mnist-1-vs-9}
\end{table}

\begin{table}[t]
\centering
\small
\setlength{\tabcolsep}{1mm}
\begin{tabular}{rccccc}
\toprule
$T$ & HLD & MMD & RootLL & Typicality & TDI \\
\midrule
\multicolumn{6}{l}{\emph{(a) \(\mathrm{FPR}_{\mathrm model}\) --- PC-vs-PC null (target $0.05$)}} \\
1 & $.040{\pm}.01$ & $.026{\pm}.01$ & $.058{\pm}.01$ & $.064{\pm}.01$ & $.050{\pm}.01$ \\
10 & $.072{\pm}.01$ & $.052{\pm}.01$ & $.070{\pm}.01$ & $.062{\pm}.01$ & $.068{\pm}.01$ \\
50 & $.058{\pm}.01$ & $.056{\pm}.01$ & $.066{\pm}.01$ & $.042{\pm}.01$ & $.070{\pm}.01$ \\
100 & $.064{\pm}.01$ & $.034{\pm}.01$ & $.038{\pm}.01$ & $.042{\pm}.01$ & $.036{\pm}.01$ \\
200 & $.058{\pm}.01$ & $.054{\pm}.01$ & $.064{\pm}.01$ & $.056{\pm}.01$ & $.058{\pm}.01$ \\
500 & $.058{\pm}.01$ & $.040{\pm}.01$ & $.054{\pm}.01$ & $.036{\pm}.01$ & $.066{\pm}.01$ \\
\bottomrule
\end{tabular}
\caption{Detection results on MNIST, pair 2-vs-3, $\alpha=0.05$.}
\label{tab:mnist-2-vs-3}
\end{table}

\begin{table}[t]
\centering
\small
\setlength{\tabcolsep}{1mm}
\begin{tabular}{rccccc}
\toprule
$T$ & HLD & MMD & RootLL & Typicality & TDI \\
\midrule
\multicolumn{6}{l}{\emph{(a) \(\mathrm{FPR}_{\mathrm model}\) --- PC-vs-PC null (target $0.05$)}} \\
1 & $.046{\pm}.01$ & $.030{\pm}.01$ & $.050{\pm}.01$ & $.048{\pm}.01$ & $.046{\pm}.01$ \\
10 & $.054{\pm}.01$ & $.078{\pm}.01$ & $.052{\pm}.01$ & $.060{\pm}.01$ & $.050{\pm}.01$ \\
50 & $.056{\pm}.01$ & $.068{\pm}.01$ & $.030{\pm}.01$ & $.054{\pm}.01$ & $.034{\pm}.01$ \\
100 & $.054{\pm}.01$ & $.032{\pm}.01$ & $.030{\pm}.01$ & $.052{\pm}.01$ & $.026{\pm}.01$ \\
200 & $.052{\pm}.01$ & $.050{\pm}.01$ & $.050{\pm}.01$ & $.048{\pm}.01$ & $.046{\pm}.01$ \\
500 & $.066{\pm}.01$ & $.060{\pm}.01$ & $.044{\pm}.01$ & $.032{\pm}.01$ & $.048{\pm}.01$ \\
\bottomrule
\end{tabular}
\caption{Detection results on MNIST, pair 2-vs-4, $\alpha=0.05$.}
\label{tab:mnist-2-vs-4}
\end{table}

\begin{table}[t]
\centering
\small
\setlength{\tabcolsep}{1mm}
\begin{tabular}{rccccc}
\toprule
$T$ & HLD & MMD & RootLL & Typicality & TDI \\
\midrule
\multicolumn{6}{l}{\emph{(a) \(\mathrm{FPR}_{\mathrm model}\) --- PC-vs-PC null (target $0.05$)}} \\
1 & $.032{\pm}.01$ & $.014{\pm}.01$ & $.056{\pm}.01$ & $.020{\pm}.01$ & $.050{\pm}.01$ \\
10 & $.058{\pm}.01$ & $.058{\pm}.01$ & $.078{\pm}.01$ & $.060{\pm}.01$ & $.074{\pm}.01$ \\
50 & $.050{\pm}.01$ & $.052{\pm}.01$ & $.054{\pm}.01$ & $.062{\pm}.01$ & $.054{\pm}.01$ \\
100 & $.064{\pm}.01$ & $.028{\pm}.01$ & $.046{\pm}.01$ & $.050{\pm}.01$ & $.048{\pm}.01$ \\
200 & $.066{\pm}.01$ & $.086{\pm}.01$ & $.046{\pm}.01$ & $.050{\pm}.01$ & $.044{\pm}.01$ \\
500 & $.080{\pm}.01$ & $.044{\pm}.01$ & $.052{\pm}.01$ & $.046{\pm}.01$ & $.058{\pm}.01$ \\
\bottomrule
\end{tabular}
\caption{Detection results on MNIST, pair 2-vs-5, $\alpha=0.05$.}
\label{tab:mnist-2-vs-5}
\end{table}

\begin{table}[t]
\centering
\small
\setlength{\tabcolsep}{1mm}
\begin{tabular}{rccccc}
\toprule
$T$ & HLD & MMD & RootLL & Typicality & TDI \\
\midrule
\multicolumn{6}{l}{\emph{(a) \(\mathrm{FPR}_{\mathrm model}\) --- PC-vs-PC null (target $0.05$)}} \\
1 & $.060{\pm}.01$ & $.020{\pm}.01$ & $.058{\pm}.01$ & $.056{\pm}.01$ & $.048{\pm}.01$ \\
10 & $.058{\pm}.01$ & $.040{\pm}.01$ & $.052{\pm}.01$ & $.062{\pm}.01$ & $.054{\pm}.01$ \\
50 & $.062{\pm}.01$ & $.064{\pm}.01$ & $.058{\pm}.01$ & $.064{\pm}.01$ & $.056{\pm}.01$ \\
100 & $.070{\pm}.01$ & $.028{\pm}.01$ & $.054{\pm}.01$ & $.056{\pm}.01$ & $.052{\pm}.01$ \\
200 & $.060{\pm}.01$ & $.080{\pm}.01$ & $.074{\pm}.01$ & $.040{\pm}.01$ & $.068{\pm}.01$ \\
500 & $.072{\pm}.01$ & $.050{\pm}.01$ & $.038{\pm}.01$ & $.030{\pm}.01$ & $.048{\pm}.01$ \\
\bottomrule
\end{tabular}
\caption{Detection results on MNIST, pair 2-vs-6, $\alpha=0.05$.}
\label{tab:mnist-2-vs-6}
\end{table}

\begin{table}[t]
\centering
\small
\setlength{\tabcolsep}{1mm}
\begin{tabular}{rccccc}
\toprule
$T$ & HLD & MMD & RootLL & Typicality & TDI \\
\midrule
\multicolumn{6}{l}{\emph{(a) \(\mathrm{FPR}_{\mathrm model}\) --- PC-vs-PC null (target $0.05$)}} \\
1 & $.040{\pm}.01$ & $.022{\pm}.01$ & $.072{\pm}.01$ & $.048{\pm}.01$ & $.068{\pm}.01$ \\
10 & $.062{\pm}.01$ & $.038{\pm}.01$ & $.062{\pm}.01$ & $.054{\pm}.01$ & $.058{\pm}.01$ \\
50 & $.054{\pm}.01$ & $.044{\pm}.01$ & $.060{\pm}.01$ & $.082{\pm}.01$ & $.058{\pm}.01$ \\
100 & $.064{\pm}.01$ & $.040{\pm}.01$ & $.034{\pm}.01$ & $.050{\pm}.01$ & $.034{\pm}.01$ \\
200 & $.050{\pm}.01$ & $.068{\pm}.01$ & $.062{\pm}.01$ & $.048{\pm}.01$ & $.062{\pm}.01$ \\
500 & $.056{\pm}.01$ & $.038{\pm}.01$ & $.058{\pm}.01$ & $.046{\pm}.01$ & $.062{\pm}.01$ \\
\bottomrule
\end{tabular}
\caption{Detection results on MNIST, pair 2-vs-7, $\alpha=0.05$.}
\label{tab:mnist-2-vs-7}
\end{table}

\begin{table}[t]
\centering
\small
\setlength{\tabcolsep}{1mm}
\begin{tabular}{rccccc}
\toprule
$T$ & HLD & MMD & RootLL & Typicality & TDI \\
\midrule
\multicolumn{6}{l}{\emph{(a) \(\mathrm{FPR}_{\mathrm model}\) --- PC-vs-PC null (target $0.05$)}} \\
1 & $.044{\pm}.01$ & $.022{\pm}.01$ & $.044{\pm}.01$ & $.038{\pm}.01$ & $.040{\pm}.01$ \\
10 & $.066{\pm}.01$ & $.048{\pm}.01$ & $.068{\pm}.01$ & $.062{\pm}.01$ & $.070{\pm}.01$ \\
50 & $.048{\pm}.01$ & $.052{\pm}.01$ & $.066{\pm}.01$ & $.080{\pm}.01$ & $.062{\pm}.01$ \\
100 & $.058{\pm}.01$ & $.030{\pm}.01$ & $.048{\pm}.01$ & $.062{\pm}.01$ & $.046{\pm}.01$ \\
200 & $.078{\pm}.01$ & $.062{\pm}.01$ & $.054{\pm}.01$ & $.044{\pm}.01$ & $.052{\pm}.01$ \\
500 & $.082{\pm}.01$ & $.038{\pm}.01$ & $.046{\pm}.01$ & $.052{\pm}.01$ & $.050{\pm}.01$ \\
\bottomrule
\end{tabular}
\caption{Detection results on MNIST, pair 2-vs-8, $\alpha=0.05$.}
\label{tab:mnist-2-vs-8}
\end{table}

\begin{table}[t]
\centering
\small
\setlength{\tabcolsep}{1mm}
\begin{tabular}{rccccc}
\toprule
$T$ & HLD & MMD & RootLL & Typicality & TDI \\
\midrule
\multicolumn{6}{l}{\emph{(a) \(\mathrm{FPR}_{\mathrm model}\) --- PC-vs-PC null (target $0.05$)}} \\
1 & $.046{\pm}.01$ & $.020{\pm}.01$ & $.064{\pm}.01$ & $.054{\pm}.01$ & $.058{\pm}.01$ \\
10 & $.068{\pm}.01$ & $.062{\pm}.01$ & $.066{\pm}.01$ & $.066{\pm}.01$ & $.070{\pm}.01$ \\
50 & $.048{\pm}.01$ & $.072{\pm}.01$ & $.038{\pm}.01$ & $.060{\pm}.01$ & $.040{\pm}.01$ \\
100 & $.042{\pm}.01$ & $.046{\pm}.01$ & $.036{\pm}.01$ & $.048{\pm}.01$ & $.032{\pm}.01$ \\
200 & $.068{\pm}.01$ & $.046{\pm}.01$ & $.062{\pm}.01$ & $.042{\pm}.01$ & $.058{\pm}.01$ \\
500 & $.074{\pm}.01$ & $.034{\pm}.01$ & $.044{\pm}.01$ & $.044{\pm}.01$ & $.050{\pm}.01$ \\
\bottomrule
\end{tabular}
\caption{Detection results on MNIST, pair 2-vs-9, $\alpha=0.05$.}
\label{tab:mnist-2-vs-9}
\end{table}

\begin{table}[t]
\centering
\small
\setlength{\tabcolsep}{1mm}
\begin{tabular}{rccccc}
\toprule
$T$ & HLD & MMD & RootLL & Typicality & TDI \\
\midrule
\multicolumn{6}{l}{\emph{(a) \(\mathrm{FPR}_{\mathrm model}\) --- PC-vs-PC null (target $0.05$)}} \\
1 & $.102{\pm}.01$ & $.030{\pm}.01$ & $.046{\pm}.01$ & $.052{\pm}.01$ & $.046{\pm}.01$ \\
10 & $.088{\pm}.01$ & $.052{\pm}.01$ & $.048{\pm}.01$ & $.054{\pm}.01$ & $.048{\pm}.01$ \\
50 & $.050{\pm}.01$ & $.066{\pm}.01$ & $.032{\pm}.01$ & $.036{\pm}.01$ & $.034{\pm}.01$ \\
100 & $.056{\pm}.01$ & $.022{\pm}.01$ & $.050{\pm}.01$ & $.056{\pm}.01$ & $.046{\pm}.01$ \\
200 & $.058{\pm}.01$ & $.046{\pm}.01$ & $.060{\pm}.01$ & $.046{\pm}.01$ & $.056{\pm}.01$ \\
500 & $.082{\pm}.01$ & $.052{\pm}.01$ & $.070{\pm}.01$ & $.100{\pm}.01$ & $.070{\pm}.01$ \\
\bottomrule
\end{tabular}
\caption{Detection results on MNIST, pair 3-vs-4, $\alpha=0.05$.}
\label{tab:mnist-3-vs-4}
\end{table}

\begin{table}[t]
\centering
\small
\setlength{\tabcolsep}{1mm}
\begin{tabular}{rccccc}
\toprule
$T$ & HLD & MMD & RootLL & Typicality & TDI \\
\midrule
\multicolumn{6}{l}{\emph{(a) \(\mathrm{FPR}_{\mathrm model}\) --- PC-vs-PC null (target $0.05$)}} \\
1 & $.106{\pm}.01$ & $.052{\pm}.01$ & $.034{\pm}.01$ & $.062{\pm}.01$ & $.038{\pm}.01$ \\
10 & $.074{\pm}.01$ & $.064{\pm}.01$ & $.022{\pm}.01$ & $.034{\pm}.01$ & $.022{\pm}.01$ \\
50 & $.060{\pm}.01$ & $.060{\pm}.01$ & $.042{\pm}.01$ & $.038{\pm}.01$ & $.052{\pm}.01$ \\
100 & $.060{\pm}.01$ & $.042{\pm}.01$ & $.068{\pm}.01$ & $.066{\pm}.01$ & $.070{\pm}.01$ \\
200 & $.082{\pm}.01$ & $.044{\pm}.01$ & $.040{\pm}.01$ & $.038{\pm}.01$ & $.040{\pm}.01$ \\
500 & $.066{\pm}.01$ & $.044{\pm}.01$ & $.044{\pm}.01$ & $.074{\pm}.01$ & $.046{\pm}.01$ \\
\bottomrule
\end{tabular}
\caption{Detection results on MNIST, pair 3-vs-5, $\alpha=0.05$.}
\label{tab:mnist-3-vs-5}
\end{table}

\begin{table}[t]
\centering
\small
\setlength{\tabcolsep}{1mm}
\begin{tabular}{rccccc}
\toprule
$T$ & HLD & MMD & RootLL & Typicality & TDI \\
\midrule
\multicolumn{6}{l}{\emph{(a) \(\mathrm{FPR}_{\mathrm model}\) --- PC-vs-PC null (target $0.05$)}} \\
1 & $.092{\pm}.01$ & $.030{\pm}.01$ & $.040{\pm}.01$ & $.060{\pm}.01$ & $.040{\pm}.01$ \\
10 & $.080{\pm}.01$ & $.066{\pm}.01$ & $.036{\pm}.01$ & $.058{\pm}.01$ & $.036{\pm}.01$ \\
50 & $.078{\pm}.01$ & $.058{\pm}.01$ & $.044{\pm}.01$ & $.052{\pm}.01$ & $.046{\pm}.01$ \\
100 & $.054{\pm}.01$ & $.056{\pm}.01$ & $.030{\pm}.01$ & $.052{\pm}.01$ & $.032{\pm}.01$ \\
200 & $.056{\pm}.01$ & $.032{\pm}.01$ & $.042{\pm}.01$ & $.044{\pm}.01$ & $.038{\pm}.01$ \\
500 & $.052{\pm}.01$ & $.026{\pm}.01$ & $.068{\pm}.01$ & $.046{\pm}.01$ & $.070{\pm}.01$ \\
\bottomrule
\end{tabular}
\caption{Detection results on MNIST, pair 3-vs-6, $\alpha=0.05$.}
\label{tab:mnist-3-vs-6}
\end{table}

\begin{table}[t]
\centering
\small
\setlength{\tabcolsep}{1mm}
\begin{tabular}{rccccc}
\toprule
$T$ & HLD & MMD & RootLL & Typicality & TDI \\
\midrule
\multicolumn{6}{l}{\emph{(a) \(\mathrm{FPR}_{\mathrm model}\) --- PC-vs-PC null (target $0.05$)}} \\
1 & $.102{\pm}.01$ & $.042{\pm}.01$ & $.044{\pm}.01$ & $.060{\pm}.01$ & $.044{\pm}.01$ \\
10 & $.058{\pm}.01$ & $.062{\pm}.01$ & $.032{\pm}.01$ & $.054{\pm}.01$ & $.040{\pm}.01$ \\
50 & $.046{\pm}.01$ & $.046{\pm}.01$ & $.052{\pm}.01$ & $.048{\pm}.01$ & $.056{\pm}.01$ \\
100 & $.062{\pm}.01$ & $.030{\pm}.01$ & $.048{\pm}.01$ & $.052{\pm}.01$ & $.048{\pm}.01$ \\
200 & $.076{\pm}.01$ & $.066{\pm}.01$ & $.060{\pm}.01$ & $.042{\pm}.01$ & $.052{\pm}.01$ \\
500 & $.046{\pm}.01$ & $.034{\pm}.01$ & $.054{\pm}.01$ & $.066{\pm}.01$ & $.054{\pm}.01$ \\
\bottomrule
\end{tabular}
\caption{Detection results on MNIST, pair 3-vs-7, $\alpha=0.05$.}
\label{tab:mnist-3-vs-7}
\end{table}

\begin{table}[t]
\centering
\small
\setlength{\tabcolsep}{1mm}
\begin{tabular}{rccccc}
\toprule
$T$ & HLD & MMD & RootLL & Typicality & TDI \\
\midrule
\multicolumn{6}{l}{\emph{(a) \(\mathrm{FPR}_{\mathrm model}\) --- PC-vs-PC null (target $0.05$)}} \\
1 & $.080{\pm}.01$ & $.032{\pm}.01$ & $.040{\pm}.01$ & $.050{\pm}.01$ & $.040{\pm}.01$ \\
10 & $.058{\pm}.01$ & $.066{\pm}.01$ & $.038{\pm}.01$ & $.036{\pm}.01$ & $.040{\pm}.01$ \\
50 & $.088{\pm}.01$ & $.046{\pm}.01$ & $.050{\pm}.01$ & $.046{\pm}.01$ & $.052{\pm}.01$ \\
100 & $.060{\pm}.01$ & $.030{\pm}.01$ & $.044{\pm}.01$ & $.060{\pm}.01$ & $.044{\pm}.01$ \\
200 & $.070{\pm}.01$ & $.052{\pm}.01$ & $.058{\pm}.01$ & $.056{\pm}.01$ & $.054{\pm}.01$ \\
500 & $.086{\pm}.01$ & $.050{\pm}.01$ & $.064{\pm}.01$ & $.076{\pm}.01$ & $.060{\pm}.01$ \\
\bottomrule
\end{tabular}
\caption{Detection results on MNIST, pair 3-vs-8, $\alpha=0.05$.}
\label{tab:mnist-3-vs-8}
\end{table}

\begin{table}[t]
\centering
\small
\setlength{\tabcolsep}{1mm}
\begin{tabular}{rccccc}
\toprule
$T$ & HLD & MMD & RootLL & Typicality & TDI \\
\midrule
\multicolumn{6}{l}{\emph{(a) \(\mathrm{FPR}_{\mathrm model}\) --- PC-vs-PC null (target $0.05$)}} \\
1 & $.108{\pm}.01$ & $.028{\pm}.01$ & $.054{\pm}.01$ & $.078{\pm}.01$ & $.050{\pm}.01$ \\
10 & $.056{\pm}.01$ & $.078{\pm}.01$ & $.038{\pm}.01$ & $.036{\pm}.01$ & $.044{\pm}.01$ \\
50 & $.060{\pm}.01$ & $.052{\pm}.01$ & $.048{\pm}.01$ & $.048{\pm}.01$ & $.052{\pm}.01$ \\
100 & $.062{\pm}.01$ & $.038{\pm}.01$ & $.034{\pm}.01$ & $.046{\pm}.01$ & $.036{\pm}.01$ \\
200 & $.040{\pm}.01$ & $.042{\pm}.01$ & $.040{\pm}.01$ & $.044{\pm}.01$ & $.036{\pm}.01$ \\
500 & $.078{\pm}.01$ & $.036{\pm}.01$ & $.060{\pm}.01$ & $.082{\pm}.01$ & $.066{\pm}.01$ \\
\bottomrule
\end{tabular}
\caption{Detection results on MNIST, pair 3-vs-9, $\alpha=0.05$.}
\label{tab:mnist-3-vs-9}
\end{table}

\begin{table}[t]
\centering
\small
\setlength{\tabcolsep}{1mm}
\begin{tabular}{rccccc}
\toprule
$T$ & HLD & MMD & RootLL & Typicality & TDI \\
\midrule
\multicolumn{6}{l}{\emph{(a) \(\mathrm{FPR}_{\mathrm model}\) --- PC-vs-PC null (target $0.05$)}} \\
1 & $.100{\pm}.01$ & $.044{\pm}.01$ & $.050{\pm}.01$ & $.060{\pm}.01$ & $.048{\pm}.01$ \\
10 & $.070{\pm}.01$ & $.056{\pm}.01$ & $.066{\pm}.01$ & $.056{\pm}.01$ & $.060{\pm}.01$ \\
50 & $.046{\pm}.01$ & $.040{\pm}.01$ & $.052{\pm}.01$ & $.036{\pm}.01$ & $.048{\pm}.01$ \\
100 & $.082{\pm}.01$ & $.066{\pm}.01$ & $.050{\pm}.01$ & $.054{\pm}.01$ & $.048{\pm}.01$ \\
200 & $.064{\pm}.01$ & $.054{\pm}.01$ & $.036{\pm}.01$ & $.052{\pm}.01$ & $.036{\pm}.01$ \\
500 & $.066{\pm}.01$ & $.056{\pm}.01$ & $.050{\pm}.01$ & $.056{\pm}.01$ & $.050{\pm}.01$ \\
\bottomrule
\end{tabular}
\caption{Detection results on MNIST, pair 4-vs-5, $\alpha=0.05$.}
\label{tab:mnist-4-vs-5}
\end{table}

\begin{table}[t]
\centering
\small
\setlength{\tabcolsep}{1mm}
\begin{tabular}{rccccc}
\toprule
$T$ & HLD & MMD & RootLL & Typicality & TDI \\
\midrule
\multicolumn{6}{l}{\emph{(a) \(\mathrm{FPR}_{\mathrm model}\) --- PC-vs-PC null (target $0.05$)}} \\
1 & $.080{\pm}.01$ & $.032{\pm}.01$ & $.048{\pm}.01$ & $.038{\pm}.01$ & $.050{\pm}.01$ \\
10 & $.076{\pm}.01$ & $.058{\pm}.01$ & $.046{\pm}.01$ & $.062{\pm}.01$ & $.044{\pm}.01$ \\
50 & $.054{\pm}.01$ & $.044{\pm}.01$ & $.052{\pm}.01$ & $.054{\pm}.01$ & $.050{\pm}.01$ \\
100 & $.074{\pm}.01$ & $.050{\pm}.01$ & $.042{\pm}.01$ & $.040{\pm}.01$ & $.036{\pm}.01$ \\
200 & $.068{\pm}.01$ & $.060{\pm}.01$ & $.036{\pm}.01$ & $.048{\pm}.01$ & $.038{\pm}.01$ \\
500 & $.048{\pm}.01$ & $.062{\pm}.01$ & $.052{\pm}.01$ & $.038{\pm}.01$ & $.054{\pm}.01$ \\
\bottomrule
\end{tabular}
\caption{Detection results on MNIST, pair 4-vs-6, $\alpha=0.05$.}
\label{tab:mnist-4-vs-6}
\end{table}

\begin{table}[t]
\centering
\small
\setlength{\tabcolsep}{1mm}
\begin{tabular}{rccccc}
\toprule
$T$ & HLD & MMD & RootLL & Typicality & TDI \\
\midrule
\multicolumn{6}{l}{\emph{(a) \(\mathrm{FPR}_{\mathrm model}\) --- PC-vs-PC null (target $0.05$)}} \\
1 & $.094{\pm}.01$ & $.046{\pm}.01$ & $.056{\pm}.01$ & $.064{\pm}.01$ & $.050{\pm}.01$ \\
10 & $.048{\pm}.01$ & $.042{\pm}.01$ & $.052{\pm}.01$ & $.048{\pm}.01$ & $.052{\pm}.01$ \\
50 & $.074{\pm}.01$ & $.034{\pm}.01$ & $.054{\pm}.01$ & $.052{\pm}.01$ & $.044{\pm}.01$ \\
100 & $.074{\pm}.01$ & $.044{\pm}.01$ & $.028{\pm}.01$ & $.054{\pm}.01$ & $.028{\pm}.01$ \\
200 & $.066{\pm}.01$ & $.056{\pm}.01$ & $.042{\pm}.01$ & $.054{\pm}.01$ & $.044{\pm}.01$ \\
500 & $.088{\pm}.01$ & $.052{\pm}.01$ & $.044{\pm}.01$ & $.050{\pm}.01$ & $.048{\pm}.01$ \\
\bottomrule
\end{tabular}
\caption{Detection results on MNIST, pair 4-vs-7, $\alpha=0.05$.}
\label{tab:mnist-4-vs-7}
\end{table}

\begin{table}[t]
\centering
\small
\setlength{\tabcolsep}{1mm}
\begin{tabular}{rccccc}
\toprule
$T$ & HLD & MMD & RootLL & Typicality & TDI \\
\midrule
\multicolumn{6}{l}{\emph{(a) \(\mathrm{FPR}_{\mathrm model}\) --- PC-vs-PC null (target $0.05$)}} \\
1 & $.114{\pm}.01$ & $.060{\pm}.01$ & $.056{\pm}.01$ & $.082{\pm}.01$ & $.058{\pm}.01$ \\
10 & $.074{\pm}.01$ & $.048{\pm}.01$ & $.044{\pm}.01$ & $.056{\pm}.01$ & $.042{\pm}.01$ \\
50 & $.074{\pm}.01$ & $.032{\pm}.01$ & $.054{\pm}.01$ & $.052{\pm}.01$ & $.052{\pm}.01$ \\
100 & $.054{\pm}.01$ & $.048{\pm}.01$ & $.034{\pm}.01$ & $.044{\pm}.01$ & $.034{\pm}.01$ \\
200 & $.072{\pm}.01$ & $.050{\pm}.01$ & $.050{\pm}.01$ & $.038{\pm}.01$ & $.050{\pm}.01$ \\
500 & $.064{\pm}.01$ & $.076{\pm}.01$ & $.036{\pm}.01$ & $.048{\pm}.01$ & $.036{\pm}.01$ \\
\bottomrule
\end{tabular}
\caption{Detection results on MNIST, pair 4-vs-8, $\alpha=0.05$.}
\label{tab:mnist-4-vs-8}
\end{table}

\begin{table}[t]
\centering
\small
\setlength{\tabcolsep}{1mm}
\begin{tabular}{rccccc}
\toprule
$T$ & HLD & MMD & RootLL & Typicality & TDI \\
\midrule
\multicolumn{6}{l}{\emph{(a) \(\mathrm{FPR}_{\mathrm model}\) --- PC-vs-PC null (target $0.05$)}} \\
1 & $.098{\pm}.01$ & $.030{\pm}.01$ & $.062{\pm}.01$ & $.040{\pm}.01$ & $.066{\pm}.01$ \\
10 & $.088{\pm}.01$ & $.062{\pm}.01$ & $.076{\pm}.01$ & $.082{\pm}.01$ & $.070{\pm}.01$ \\
50 & $.076{\pm}.01$ & $.062{\pm}.01$ & $.050{\pm}.01$ & $.054{\pm}.01$ & $.052{\pm}.01$ \\
100 & $.070{\pm}.01$ & $.056{\pm}.01$ & $.042{\pm}.01$ & $.056{\pm}.01$ & $.042{\pm}.01$ \\
200 & $.060{\pm}.01$ & $.078{\pm}.01$ & $.034{\pm}.01$ & $.044{\pm}.01$ & $.032{\pm}.01$ \\
500 & $.064{\pm}.01$ & $.056{\pm}.01$ & $.046{\pm}.01$ & $.040{\pm}.01$ & $.046{\pm}.01$ \\
\bottomrule
\end{tabular}
\caption{Detection results on MNIST, pair 4-vs-9, $\alpha=0.05$.}
\label{tab:mnist-4-vs-9}
\end{table}

\begin{table}[t]
\centering
\small
\setlength{\tabcolsep}{1mm}
\begin{tabular}{rccccc}
\toprule
$T$ & HLD & MMD & RootLL & Typicality & TDI \\
\midrule
\multicolumn{6}{l}{\emph{(a) \(\mathrm{FPR}_{\mathrm model}\) --- PC-vs-PC null (target $0.05$)}} \\
1 & $.094{\pm}.01$ & $.056{\pm}.01$ & $.050{\pm}.01$ & $.050{\pm}.01$ & $.050{\pm}.01$ \\
10 & $.062{\pm}.01$ & $.058{\pm}.01$ & $.062{\pm}.01$ & $.062{\pm}.01$ & $.060{\pm}.01$ \\
50 & $.060{\pm}.01$ & $.054{\pm}.01$ & $.060{\pm}.01$ & $.038{\pm}.01$ & $.058{\pm}.01$ \\
100 & $.072{\pm}.01$ & $.068{\pm}.01$ & $.054{\pm}.01$ & $.074{\pm}.01$ & $.052{\pm}.01$ \\
200 & $.056{\pm}.01$ & $.034{\pm}.01$ & $.056{\pm}.01$ & $.046{\pm}.01$ & $.046{\pm}.01$ \\
500 & $.068{\pm}.01$ & $.022{\pm}.01$ & $.054{\pm}.01$ & $.052{\pm}.01$ & $.058{\pm}.01$ \\
\bottomrule
\end{tabular}
\caption{Detection results on MNIST, pair 5-vs-6, $\alpha=0.05$.}
\label{tab:mnist-5-vs-6}
\end{table}

\begin{table}[t]
\centering
\small
\setlength{\tabcolsep}{1mm}
\begin{tabular}{rccccc}
\toprule
$T$ & HLD & MMD & RootLL & Typicality & TDI \\
\midrule
\multicolumn{6}{l}{\emph{(a) \(\mathrm{FPR}_{\mathrm model}\) --- PC-vs-PC null (target $0.05$)}} \\
1 & $.112{\pm}.01$ & $.044{\pm}.01$ & $.040{\pm}.01$ & $.044{\pm}.01$ & $.042{\pm}.01$ \\
10 & $.050{\pm}.01$ & $.036{\pm}.01$ & $.070{\pm}.01$ & $.060{\pm}.01$ & $.066{\pm}.01$ \\
50 & $.070{\pm}.01$ & $.054{\pm}.01$ & $.054{\pm}.01$ & $.032{\pm}.01$ & $.060{\pm}.01$ \\
100 & $.046{\pm}.01$ & $.056{\pm}.01$ & $.054{\pm}.01$ & $.046{\pm}.01$ & $.054{\pm}.01$ \\
200 & $.056{\pm}.01$ & $.046{\pm}.01$ & $.074{\pm}.01$ & $.074{\pm}.01$ & $.072{\pm}.01$ \\
500 & $.054{\pm}.01$ & $.044{\pm}.01$ & $.046{\pm}.01$ & $.046{\pm}.01$ & $.040{\pm}.01$ \\
\bottomrule
\end{tabular}
\caption{Detection results on MNIST, pair 5-vs-7, $\alpha=0.05$.}
\label{tab:mnist-5-vs-7}
\end{table}

\begin{table}[t]
\centering
\small
\setlength{\tabcolsep}{1mm}
\begin{tabular}{rccccc}
\toprule
$T$ & HLD & MMD & RootLL & Typicality & TDI \\
\midrule
\multicolumn{6}{l}{\emph{(a) \(\mathrm{FPR}_{\mathrm model}\) --- PC-vs-PC null (target $0.05$)}} \\
1 & $.078{\pm}.01$ & $.054{\pm}.01$ & $.060{\pm}.01$ & $.040{\pm}.01$ & $.062{\pm}.01$ \\
10 & $.066{\pm}.01$ & $.046{\pm}.01$ & $.060{\pm}.01$ & $.046{\pm}.01$ & $.058{\pm}.01$ \\
50 & $.072{\pm}.01$ & $.064{\pm}.01$ & $.060{\pm}.01$ & $.036{\pm}.01$ & $.060{\pm}.01$ \\
100 & $.078{\pm}.01$ & $.062{\pm}.01$ & $.074{\pm}.01$ & $.068{\pm}.01$ & $.066{\pm}.01$ \\
200 & $.052{\pm}.01$ & $.046{\pm}.01$ & $.050{\pm}.01$ & $.040{\pm}.01$ & $.042{\pm}.01$ \\
500 & $.080{\pm}.01$ & $.042{\pm}.01$ & $.042{\pm}.01$ & $.036{\pm}.01$ & $.038{\pm}.01$ \\
\bottomrule
\end{tabular}
\caption{Detection results on MNIST, pair 5-vs-8, $\alpha=0.05$.}
\label{tab:mnist-5-vs-8}
\end{table}

\begin{table}[t]
\centering
\small
\setlength{\tabcolsep}{1mm}
\begin{tabular}{rccccc}
\toprule
$T$ & HLD & MMD & RootLL & Typicality & TDI \\
\midrule
\multicolumn{6}{l}{\emph{(a) \(\mathrm{FPR}_{\mathrm model}\) --- PC-vs-PC null (target $0.05$)}} \\
1 & $.074{\pm}.01$ & $.040{\pm}.01$ & $.040{\pm}.01$ & $.050{\pm}.01$ & $.052{\pm}.01$ \\
10 & $.054{\pm}.01$ & $.054{\pm}.01$ & $.042{\pm}.01$ & $.042{\pm}.01$ & $.038{\pm}.01$ \\
50 & $.042{\pm}.01$ & $.056{\pm}.01$ & $.064{\pm}.01$ & $.048{\pm}.01$ & $.062{\pm}.01$ \\
100 & $.058{\pm}.01$ & $.052{\pm}.01$ & $.050{\pm}.01$ & $.056{\pm}.01$ & $.046{\pm}.01$ \\
200 & $.052{\pm}.01$ & $.054{\pm}.01$ & $.038{\pm}.01$ & $.056{\pm}.01$ & $.034{\pm}.01$ \\
500 & $.074{\pm}.01$ & $.044{\pm}.01$ & $.046{\pm}.01$ & $.050{\pm}.01$ & $.044{\pm}.01$ \\
\bottomrule
\end{tabular}
\caption{Detection results on MNIST, pair 5-vs-9, $\alpha=0.05$.}
\label{tab:mnist-5-vs-9}
\end{table}

\begin{table}[t]
\centering
\small
\setlength{\tabcolsep}{1mm}
\begin{tabular}{rccccc}
\toprule
$T$ & HLD & MMD & RootLL & Typicality & TDI \\
\midrule
\multicolumn{6}{l}{\emph{(a) \(\mathrm{FPR}_{\mathrm model}\) --- PC-vs-PC null (target $0.05$)}} \\
1 & $.088{\pm}.01$ & $.016{\pm}.01$ & $.048{\pm}.01$ & $.032{\pm}.01$ & $.048{\pm}.01$ \\
10 & $.068{\pm}.01$ & $.054{\pm}.01$ & $.070{\pm}.01$ & $.064{\pm}.01$ & $.070{\pm}.01$ \\
50 & $.048{\pm}.01$ & $.028{\pm}.01$ & $.048{\pm}.01$ & $.030{\pm}.01$ & $.046{\pm}.01$ \\
100 & $.058{\pm}.01$ & $.056{\pm}.01$ & $.038{\pm}.01$ & $.050{\pm}.01$ & $.036{\pm}.01$ \\
200 & $.068{\pm}.01$ & $.028{\pm}.01$ & $.058{\pm}.01$ & $.044{\pm}.01$ & $.058{\pm}.01$ \\
500 & $.076{\pm}.01$ & $.038{\pm}.01$ & $.054{\pm}.01$ & $.052{\pm}.01$ & $.052{\pm}.01$ \\
\bottomrule
\end{tabular}
\caption{Detection results on MNIST, pair 6-vs-7, $\alpha=0.05$.}
\label{tab:mnist-6-vs-7}
\end{table}

\begin{table}[t]
\centering
\small
\setlength{\tabcolsep}{1mm}
\begin{tabular}{rccccc}
\toprule
$T$ & HLD & MMD & RootLL & Typicality & TDI \\
\midrule
\multicolumn{6}{l}{\emph{(a) \(\mathrm{FPR}_{\mathrm model}\) --- PC-vs-PC null (target $0.05$)}} \\
1 & $.088{\pm}.01$ & $.016{\pm}.01$ & $.042{\pm}.01$ & $.038{\pm}.01$ & $.042{\pm}.01$ \\
10 & $.060{\pm}.01$ & $.058{\pm}.01$ & $.056{\pm}.01$ & $.060{\pm}.01$ & $.058{\pm}.01$ \\
50 & $.054{\pm}.01$ & $.044{\pm}.01$ & $.050{\pm}.01$ & $.040{\pm}.01$ & $.054{\pm}.01$ \\
100 & $.050{\pm}.01$ & $.072{\pm}.01$ & $.054{\pm}.01$ & $.054{\pm}.01$ & $.052{\pm}.01$ \\
200 & $.064{\pm}.01$ & $.052{\pm}.01$ & $.038{\pm}.01$ & $.026{\pm}.01$ & $.038{\pm}.01$ \\
500 & $.064{\pm}.01$ & $.048{\pm}.01$ & $.060{\pm}.01$ & $.070{\pm}.01$ & $.056{\pm}.01$ \\
\bottomrule
\end{tabular}
\caption{Detection results on MNIST, pair 6-vs-8, $\alpha=0.05$.}
\label{tab:mnist-6-vs-8}
\end{table}

\begin{table}[t]
\centering
\small
\setlength{\tabcolsep}{1mm}
\begin{tabular}{rccccc}
\toprule
$T$ & HLD & MMD & RootLL & Typicality & TDI \\
\midrule
\multicolumn{6}{l}{\emph{(a) \(\mathrm{FPR}_{\mathrm model}\) --- PC-vs-PC null (target $0.05$)}} \\
1 & $.096{\pm}.01$ & $.022{\pm}.01$ & $.044{\pm}.01$ & $.044{\pm}.01$ & $.044{\pm}.01$ \\
10 & $.076{\pm}.01$ & $.056{\pm}.01$ & $.032{\pm}.01$ & $.054{\pm}.01$ & $.034{\pm}.01$ \\
50 & $.064{\pm}.01$ & $.046{\pm}.01$ & $.034{\pm}.01$ & $.058{\pm}.01$ & $.038{\pm}.01$ \\
100 & $.052{\pm}.01$ & $.050{\pm}.01$ & $.044{\pm}.01$ & $.046{\pm}.01$ & $.038{\pm}.01$ \\
200 & $.068{\pm}.01$ & $.048{\pm}.01$ & $.052{\pm}.01$ & $.042{\pm}.01$ & $.050{\pm}.01$ \\
500 & $.054{\pm}.01$ & $.056{\pm}.01$ & $.062{\pm}.01$ & $.062{\pm}.01$ & $.066{\pm}.01$ \\
\bottomrule
\end{tabular}
\caption{Detection results on MNIST, pair 6-vs-9, $\alpha=0.05$.}
\label{tab:mnist-6-vs-9}
\end{table}

\begin{table}[t]
\centering
\small
\setlength{\tabcolsep}{1mm}
\begin{tabular}{rccccc}
\toprule
$T$ & HLD & MMD & RootLL & Typicality & TDI \\
\midrule
\multicolumn{6}{l}{\emph{(a) \(\mathrm{FPR}_{\mathrm model}\) --- PC-vs-PC null (target $0.05$)}} \\
1 & $.076{\pm}.01$ & $.052{\pm}.01$ & $.038{\pm}.01$ & $.044{\pm}.01$ & $.040{\pm}.01$ \\
10 & $.058{\pm}.01$ & $.054{\pm}.01$ & $.108{\pm}.01$ & $.074{\pm}.01$ & $.098{\pm}.01$ \\
50 & $.084{\pm}.01$ & $.052{\pm}.01$ & $.052{\pm}.01$ & $.040{\pm}.01$ & $.052{\pm}.01$ \\
100 & $.066{\pm}.01$ & $.048{\pm}.01$ & $.048{\pm}.01$ & $.034{\pm}.01$ & $.052{\pm}.01$ \\
200 & $.080{\pm}.01$ & $.044{\pm}.01$ & $.046{\pm}.01$ & $.064{\pm}.01$ & $.042{\pm}.01$ \\
500 & $.072{\pm}.01$ & $.040{\pm}.01$ & $.036{\pm}.01$ & $.030{\pm}.01$ & $.040{\pm}.01$ \\
\bottomrule
\end{tabular}
\caption{Detection results on MNIST, pair 7-vs-8, $\alpha=0.05$.}
\label{tab:mnist-7-vs-8}
\end{table}

\begin{table}[t]
\centering
\small
\setlength{\tabcolsep}{1mm}
\begin{tabular}{rccccc}
\toprule
$T$ & HLD & MMD & RootLL & Typicality & TDI \\
\midrule
\multicolumn{6}{l}{\emph{(a) \(\mathrm{FPR}_{\mathrm model}\) --- PC-vs-PC null (target $0.05$)}} \\
1 & $.084{\pm}.01$ & $.038{\pm}.01$ & $.058{\pm}.01$ & $.054{\pm}.01$ & $.060{\pm}.01$ \\
10 & $.070{\pm}.01$ & $.052{\pm}.01$ & $.090{\pm}.01$ & $.048{\pm}.01$ & $.080{\pm}.01$ \\
50 & $.058{\pm}.01$ & $.042{\pm}.01$ & $.072{\pm}.01$ & $.060{\pm}.01$ & $.066{\pm}.01$ \\
100 & $.074{\pm}.01$ & $.060{\pm}.01$ & $.056{\pm}.01$ & $.054{\pm}.01$ & $.058{\pm}.01$ \\
200 & $.066{\pm}.01$ & $.048{\pm}.01$ & $.050{\pm}.01$ & $.048{\pm}.01$ & $.046{\pm}.01$ \\
500 & $.070{\pm}.01$ & $.058{\pm}.01$ & $.044{\pm}.01$ & $.038{\pm}.01$ & $.042{\pm}.01$ \\
\bottomrule
\end{tabular}
\caption{Detection results on MNIST, pair 7-vs-9, $\alpha=0.05$.}
\label{tab:mnist-7-vs-9}
\end{table}

\begin{table}[t]
\centering
\small
\setlength{\tabcolsep}{1mm}
\begin{tabular}{rccccc}
\toprule
$T$ & HLD & MMD & RootLL & Typicality & TDI \\
\midrule
\multicolumn{6}{l}{\emph{(a) \(\mathrm{FPR}_{\mathrm model}\) --- PC-vs-PC null (target $0.05$)}} \\
1 & $.122{\pm}.01$ & $.052{\pm}.01$ & $.052{\pm}.01$ & $.070{\pm}.01$ & $.050{\pm}.01$ \\
10 & $.070{\pm}.01$ & $.038{\pm}.01$ & $.048{\pm}.01$ & $.062{\pm}.01$ & $.048{\pm}.01$ \\
50 & $.054{\pm}.01$ & $.046{\pm}.01$ & $.046{\pm}.01$ & $.080{\pm}.01$ & $.046{\pm}.01$ \\
100 & $.076{\pm}.01$ & $.040{\pm}.01$ & $.054{\pm}.01$ & $.064{\pm}.01$ & $.058{\pm}.01$ \\
200 & $.050{\pm}.01$ & $.030{\pm}.01$ & $.052{\pm}.01$ & $.060{\pm}.01$ & $.048{\pm}.01$ \\
500 & $.072{\pm}.01$ & $.062{\pm}.01$ & $.050{\pm}.01$ & $.058{\pm}.01$ & $.046{\pm}.01$ \\
\bottomrule
\end{tabular}
\caption{Detection results on MNIST, pair 8-vs-9, $\alpha=0.05$.}
\label{tab:mnist-8-vs-9}
\end{table}

\end{document}